\documentclass[11pt]{article}
\usepackage[margin=1in]{geometry}

\usepackage[utf8]{inputenc} 
\usepackage[T1]{fontenc}    
\usepackage{booktabs}       
\usepackage{amsfonts}       
\usepackage{nicefrac}       
\usepackage{microtype}      
 
\newcommand{\mat}[1]{\mathbf{#1}}
\usepackage{color}
\usepackage{mathtools}

\usepackage{graphicx} 
\usepackage{subcaption} 

\usepackage{amsmath}
\usepackage{amsthm,amssymb}
\usepackage{enumitem}

\usepackage{authblk} 

\usepackage{algorithm}
\usepackage{algorithmic}

\newcommand{\pr}[1]{\mathbb{P}(#1)}

\newcommand{\E}{\mathbb{E}}

\newtheorem{theorem}{Theorem}[section]
\newtheorem{lemma}[theorem]{Lemma}
\newtheorem{definition}{Definition}
\newtheorem{problem}{Problem}
\newtheorem{proposition}{Proposition}

\newcommand*\samethanks[1][\value{footnote}]{\footnotemark[#1]}

\begin{document} 
\date{\today}
\title{Sparse Quadratic Logistic Regression in Sub-quadratic Time}
\author[1]{Karthikeyan Shanmugam\thanks{Equal contribution.}}
\author[2]{Murat Kocaoglu\samethanks}
\author[2]{Alexandros G. Dimakis}
\author[2]{Sujay Sanghavi}
\affil[1]{\small IBM Research, T. J. Watson Center}
\affil[2]{\small The University of Texas at Austin}
\affil[ ]{\small   karthikeyan.shanmugam2@ibm.com,  mkocaoglu@utexas.edu, dimakis@austin.utexas.edu, sanghavi@mail.utexas.edu}
\renewcommand\Authands{ and }

\maketitle

\begin{abstract} 
We consider support recovery in the {\em quadratic} logistic regression setting -- where the target depends on both $p$ linear terms $x_i$ and up to $p^2$ quadratic terms $x_i x_j$. Quadratic terms enable prediction/modeling of higher-order effects between features and the target, but when incorporated naively may involve solving a very large regression problem. We consider the {\em sparse} case, where at most $s$ terms (linear or quadratic) are non-zero, and provide a new faster  algorithm. It involves {\em (a)} identifying the weak support (i.e. all relevant variables) and {\em (b)} standard logistic regression optimization only on these chosen variables. The first step relies on a novel insight about correlation tests in the presence of non-linearity, and takes $O(pn)$ time for $n$ samples -- giving potentially huge computational gains over the naive approach. Motivated by insights from the boolean case, we propose a non-linear correlation test for non-binary finite support case that involves hashing a variable and then correlating with the output variable. We also provide experimental results to demonstrate the effectiveness of our methods.
\end{abstract} 

\section{Introduction}
In this paper we consider the following simple {\bf quadratic logistic regression setting}: we have $p$ predictor variables $X_i \in \mathcal{X} \subset \mathbb{R}$ each with finite support on the real line($\lvert  \mathcal{X} \rvert$ is finite), with $i=1,\ldots,p$, and one binary target variable $Y\in \{0,1\}$ which depends on $X$ as follows:
\begin{align}
\label{def:Y}
&\Pr(Y=1| \mathbf{X}=\mathbf{x}) = \sigma (\gamma f(\mathbf{x}))  = \frac{\exp(\gamma f(\mathbf{x}))}{1+\exp(\gamma f(\mathbf{x}))},\nonumber\\
&\text{where } f(x)=\sum_{ (i,j) \in Q} \beta_{i,j} x_i x_j + \sum_{j\in L} \alpha_j x_j + c
\end{align}

Note that, here we have indexed the quadratic terms by $Q \subseteq [p] \times [p]$ and the linear terms by $L \subseteq [p]$. Here, $\sigma(\cdot)$ is a monotone non-linear function whose output range is in $[0,1]$ and $\gamma$ is a rescaling constant. As a specific example, we can consider the non-linear function $\sigma$ to be a logistic function $\sigma(\cdot)= \frac{\exp(\cdot)}{1+\exp(\cdot)}$, but our results hold more generally. Here, $\gamma \in \mathbb{R}^{+}$ is a scaling parameter. 

Let the coefficients be bounded, i.e., $0<a \leq \lvert \beta_{i,j} \rvert,~ \lvert \alpha_j \rvert,~ \lvert c \rvert \leq b $. We assume that the function $f$ is \textbf{$s$-sparse}, i.e. $ \lvert Q \rvert + \lvert L \rvert \leq s$. The number of variables that the function depends on is $s \leq r \leq 2s$. 
This problem is called a sparse logistic regression with real features with finite support.  For example, $y$ can indicate if a user will click on a displayed advertisement and $x_i$ can be known user features. The function $f(\mat{x})$ captures how desirable an ad is to a user (each ad will be described by its own polynomial coefficients $\alpha,\beta$) and the non-linear function translates that into a probability of clicking which is then observed through a Bernoulli random variable $y$. The problem is to learn the polynomial coefficients $\alpha,\beta$ from samples $(\mat{x},y)$ generated from some known distribution. 

The well-known challenge with learning higher-order polynomials is the sample and the computational complexity.
The model has $\Theta(p^2)$ coefficients that must be learned, and the number of features $p$ can be in the order of thousands or even millions for many applications. The key simplifying assumption is that the number of actually influential features (and influential pairs of features) can be small.  
We therefore assume that the total number of non-zero linear and bi-linear terms in $f(\mat{x})$ is at most $s$,
\textit{i.e.} $||\alpha||_0 + ||\beta||_0 \leq s$.

Several methods have been used for sparse logistic regression, including 
maximum likelihood estimation with regularization and greedy methods~\cite{ng2004feature,koh2007interior,lee2006efficient}.
All previously known approaches for this problem have a running time that is $\Theta(p^2n)$, i.e. quadratic in the ambient dimension $p$. 

\subsection{Our Contributions}
\textbf{Algorithm:} We propose two algorithms for sparse quadratic support recovery for logistic regression, one for binary input variables and the other for bounded non-binary real valued variables. The algorithm for the binary case is {\em (a)} a simple correlation test per variable that identifies relevant variables (termed the {\em weak support}) and {\em (b)} a standard ML optimization on the set of all linear and quadratic terms formed by the weak support. The first step takes $O(pn)$ time for $p$ variables and $n$ samples; the second step is standard logistic regression with $O(w^2)$ features, where $w$ is the size of the weak support. The {\bf main insight} here is that {\em except for a measure-zero set of coefficients}, standard correlation tests recover the weak support even for quadratic terms -- but only when there are non-linearities (like the sigmoid function in logistic regression).  For the case of non-binary real valued variables with finite support, we propose a generalized non-linear correlation test for weak support recovery. First, a candidate variable is hashed using a random hash function and a correlation test performed to recover the weak support. The runtime complexity of the first step in this case is also $O(pn)$.

\textbf{Analysis:} In the binary case, we show that a simple linear correlation test is asymptotically effective when the non-linear function $\sigma(\cdot)$ is \textit{strongly non-linear}, i.e. it has $2^s$ non-zero terms in its Taylor expansion, and the polynomial coefficients are not in a set of measure 0 which we precisely characterize. We show finite sample complexity results for a specific non-linearity (piecewise-linear $\sigma(\cdot)$). In short, the fundamental property that is key to our results is the non-linearity $\sigma(\cdot)$ and an extra degree of freedom $\gamma$. Insights about non-linearity plays a strong role behind our non-linear correlation test for real variables with finite support. This test introduces non-linearity by transforming the features using a hash function before calculating the correlation with the target variable.

\textbf{Experiments:} We show that linear correlation test works on synthetic randomly chosen functions with tens of relevant variables and $1000$s of irrelevant variables. With tens of thousands of samples, it recovers $>90\%$ of the weak support. We also show how correlation test and addition of few quadratic terms based on the weak support to the set of features helps improve standard classifiers on a real data set. We also similar positive experimental results for the non-linear correlation test in the non-binary case.
 
\subsection{Related Work}
Logistic regression is a very well studied problem in statistics. Logistic regression involving a linear function of the variables is carried out by computing a maximum-likelihood estimator which is known to be convex. If the number of variables is small, i.e., if we have a sparse logistic regression problem, then $\ell_1$ regularization (sometimes other regularizers) is used to enforce sparsity. Obtaining faster convex optimization based solvers for this particular convex optimization problem is the focus for many prior studies (see \cite{lee2006efficient}\cite{koh2007interior} and \cite{friedman2010regularization}). In the online setting, \cite{agarwal2014scalable} consider the problem of adaptively including higher order monomials when using online convex optimization tools for regression problems on polynomials.

Another important classic problem is learning Ising models with pairwise interactions. Learning pairwise Ising Models on $p$ nodes can be thought of as $p$ parallel logistic regression problems involving a linear functional.  Ravikumar et al. \cite{Ravikumar10} propose convex optimization methods to solve these problems while Bresler et al. \cite{Bresler15} showed that simple correlation tests and a greedy procedure is sufficient. The second work is closer to ours in spirit. However, we work with quadratic polynomials of independent variables with finite support. 

Another line of work in statistics literature \cite{li2016robust,fan2015innovated} investigates the problem of identifying relevant linear and quadratic terms in a logistic regression problem in high dimensional settings with real covariates. \cite{fan2015innovated} propose a screening procedure to identify relevant variables. However, they consider a Gaussian mixture model and their method needs estimation of precision matrices that has quadratic runtime complexity. The recent work of \cite{li2016robust} proposes a greedy selection procedure for interaction terms and linear terms. At every step, their method involves evaluating the change in the logistic likelihood by adding every possible variable outside the current set and all the relevant quadratic terms. This has linear complexity per step. The evaluation of logistic likelihood has to be done on a growing set of variables which is also expensive. Contrary to these prior works, in this paper, we focus on real variables with finite support and achieve a total running time that is linear in the number of variables.

There has been a lot of work analyzing the time and sample complexity requirements for learning sparse polynomials over the boolean domain when the evaluations of the polynomial over the samples are directly observed. The closest to our work are \cite{li2015active,kocaoglu2014sparse}. In \cite{kocaoglu2014sparse}, the authors study the general problem of learning a sparse polynomial function from random boolean samples. They define the unique sign property (USP) property which is also key for our machinery. However, it has several key differences since it does not achieve sub-quadratic learning and only learns from samples that are uniform on the hypercube. The authors in \cite{li2015active} deal with the problem of adaptive learning of a sparse polynomial where the algorithm has the power to query a specific sample. In learning theory, there is a lot of work  \cite{kalai2009learning,valiant2012finding,mossel2003learning,feldman2006new} on learning real valued sparse polynomials where the polynomial is evaluated on boolean samples from a product distribution. The closest to our work is \cite{valiant2012finding} that gives a sub quadratic algorithm for learning any sparse quadratic polynomial of independent boolean variables with biases $1/2$ in sub-quadratic time. In our case, the setup involves a non-linearity and arbitrary biases of boolean variables although bounded away from $1/2$. 

\section{Problem Statement}
 In this paper we consider the sparse quadratic logistic regression setting as in (\ref{def:Y}). The main question we are interested in is:
\begin{problem}
 Given samples $\langle y_i,\mathbf{x}_i \rangle$ distributed according to the logistic regression model in $\ref{def:Y}$, when $f$ is a quadratic polynomial containing at most $s$ parities, then is it possible to learn the function $f$ in \textbf{sub-quadratic time $\mathbf{\tilde{o}(p^2)}$ }? If possible, is there an algorithm that runs in linear time $\tilde{O}(p)$?
\end{problem}
$\tilde{o},\tilde{O}$ are equivalent the usual $o,O$ notation except in that it could include terms that are independent of the dimension $p$ (like sparsity, number of samples etc.). In this paper, we consider this question in the simplest case, when the $X_i$'s are independent. For the boolean case where $\mathcal{X} \in \{+1,1\}$ we assume $\Pr[X_i=1]=p_i$. Otherwise, we assume that $X_i$ is distributed according to a discrete pmf on the support $\mathcal{X}$.

Since there are $\Theta(p^2)$ possible choices of non-zero coefficients, a log-likelihood optimization algorithm requires runtime that is polynomial in $p^2$ for estimating all the coefficients. 

\emph{A correlation test} with a single variable $X_i$ checks if $\lvert \mathbb{E}[Y(X_i-\mu_i)] \rvert > \epsilon$ or not by using the empirical mean to calculate the correlation value. In this work, we discover a surprising property of this simple test in relation to logistic regression on binary variables: We first note that if the logistic function is not present and when the function is observed directly, i.e., $Y = f(\mathbf{X})$, then the weak-support cannot be detected from correlation of $Y$ with individual binary variables. We discuss this in the next section and show how the non-linearity $\sigma$ helps in identifying relevant variables even in cases when it is not possible to identify using correlation tests.

\section{Simple Example: Non-linearity helps in the Boolean case}
Let $\mathcal {X} \in \{+1,-1\}$. We now show that having the both the non-linearity $\sigma(\cdot)$, and biases of the input binary variables that is {\em not} either 0,1 or 1/2 exactly, is crucial to correlation tests (and hence also our algorithm) to work. We demonstrate this through a simple example and follow it up with an informal discussion on why it happens.
Suppose we consider the problem of identifying the relevant binary variables $x_i,x_j$ out of $p$ dependent variables given the labeled samples $(Y,\mathbf{x})$, where $Y=f(\mathbf{x})$. Let $f$ be a simple quadratic function $C(x_i - \mu_i) (x_j - \mu_j)$. Clearly, all correlation tests of the form $\E \left[ Y (x_k - \mu_k) \right] $, where $\mu_k=\E[x_k]$ are identically $0$ for all $k \in [p]$. Hence, correlation tests cannot be used to identify $x_i$ or $x_j$. We are not aware of an algorithm to identify the weak support in linear time for these classes of functions.

Now, consider the same function $f$ with the logistic regression setup in $(\ref{def:Y})$. Let $\gamma=1$. It can be shown that $\E \left[ Y (x_k - \mu_k) \right] = \E \left[ \sigma(f) \left( x_k - \mu_k \right) \right] \neq 0$ for $k={i,j}$ and $0$ otherwise, when $p_i, p_j $ are \textit{away} from $0,1$ and $1/2$. Except in the non-trivial case when the variables are close to being uniform Bernoulli random variables, the correlation test works. Let $x_1,x_2$ be the two variables in $f(\mat{x})$. We plot $\E[\sigma(f)(X_1-\mu_1)]$ for $C=20, \gamma=1$ for all $p_1,p_2 \in [0,1]$ in Fig. \ref{Graphcorr}. It is clear that, when $p_1,p_2 \neq \{0,1,1/2\}$, the correlation is non-zero. If the distance from these points is larger, then the correlation is also higher.

 \begin{figure}[ht]
\centering
\includegraphics[width=9.5cm]{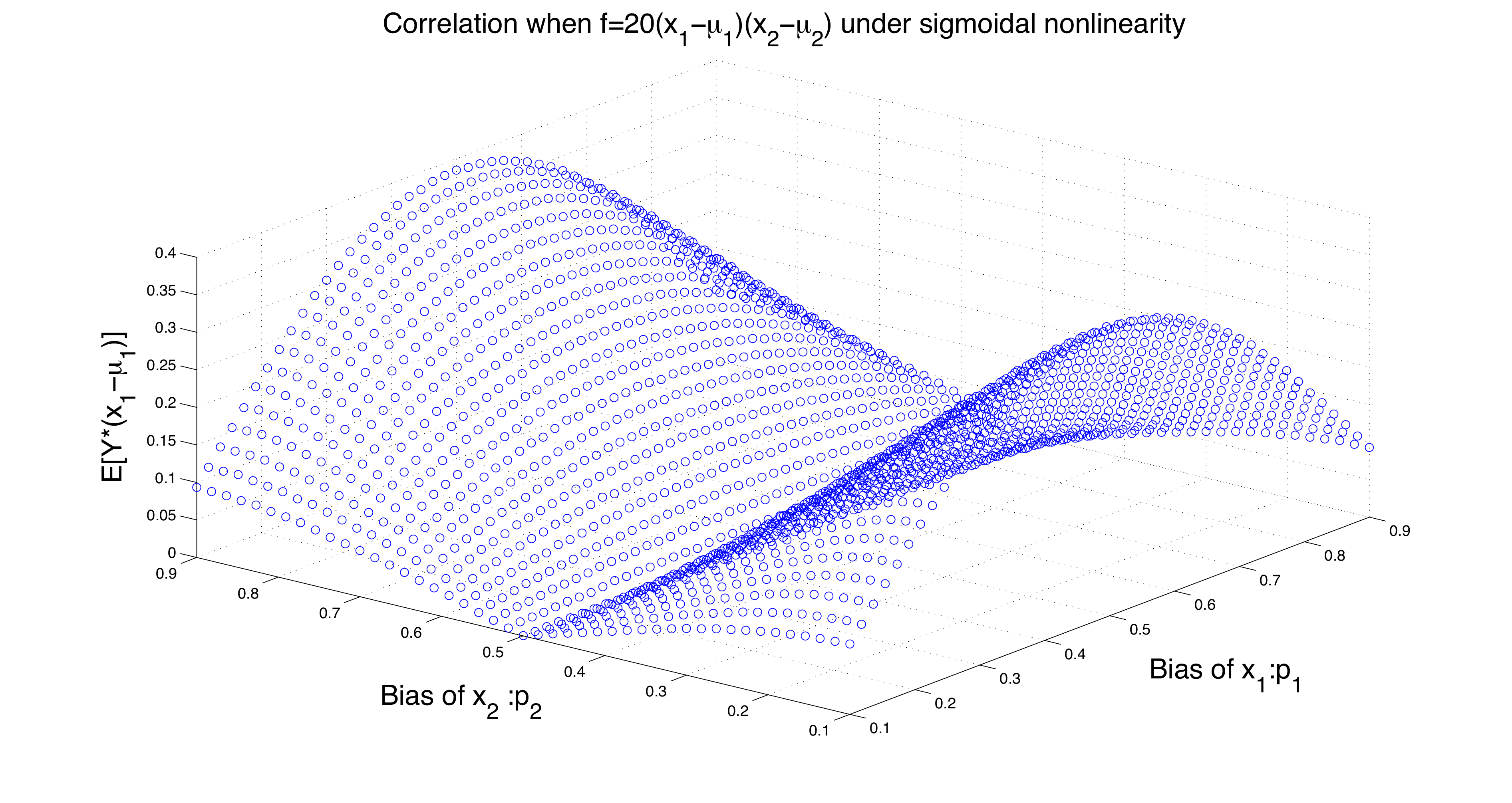}
  \caption{\small Plot of $\E[Y(X_1-\mu_1)]$ 
  when $Y\sim \sigma(f)$, $f=20(x_1-\mu_1)(x_2-\mu_2)$ and non-linearity $\sigma(f)=\frac{\exp(f)}{1+\exp(f)}$. Notice that non-linearity induces correlation for when $p_1,p_2$ is away from $0,1$ and $1/2$, whereas a linear $\sigma(.)$ would yield zero correlation irrespective of the bias probabilities.}
\label{Graphcorr}
\end{figure}

The quadratic polynomial $f$ is passed through a non-linear function $\sigma(\cdot)$ and we get to observe a noisy version of it (a coin toss with that probability). Here, the non-linearity $\sigma$ \textit{helps} identify the variables using a simple correlation test. If $\sigma (\cdot)$ is linear, as in the above example, it is not possible to detect using a correlation test. This is the key message. We briefly explore \textit{why} non-linearity helps and why it would not work with uniform random variables.  In the next section, we show that for any set of bias probabilities away from $0,1$ and $1/2$, the correlation test succeeds when the unique sign property holds for the polynomial $f$ almost surely for scaling parameter $\gamma \in [0,1]$.

\subsection{Key Idea behind our results: Why does non-linearity help?}\label{whyhelp}
Suppose $\sigma(x)$ is the polynomial $x^q$ for some positive integer $q$. Let $v_1 \ldots v_m$ be the distinct non-zero values taken by $f$. Consider the \textit{influence} of a relevant variable $x_i$ on the value $v_i$: $g(v_i)=\mathrm{Pr}(\mathbf{x}:  f(\mathbf{x})=v_i \lvert x_i=1 )- \mathrm{Pr}(\mathbf{x}:f(\mathbf{x})=v_i \lvert  x_i=-1)$. For an irrelevant variable, all influences are $0$. Then the correlation with $(x_i -\mu_i)$ can be shown to be proportional to: $\gamma^q\sum v_i^q g(v_i)$. If the test fails, then this specific linear combination of influences is $0$. Therefore, one can have non-zero influences and still fail the test. 

 Let us consider a non-linearity $\sigma(\cdot)$ that is analytic over the real line possessing an infinite degree taylor expansion. Now, the correlation with $(x_i-\mu_i)$ is an infinite degree polynomial in $\gamma$, and for the test to fail for all $\gamma$, every coefficient has to be zero or the test fails for only countable values of $\gamma$. The coefficient of the $q$-th degree term is indeed $\sum v_i^q g(v_i)$. This means that the influence variables $g(v_i)$ are constrained by infinite equations, one for every degree $q$. For the test to fail for all $\gamma$, all these equations must be identically $0$.
 
 The unique sign property makes sure that $m$ of those equations are full rank and that there exists an influence variable that is not identically zero. Therefore, the test cannot fail for all $\gamma$. The presence of a strong non-linearity and the scaling parameter (along with the USP property) forces \textit{all} influences to zero which is not possible. We show this formally in Theorem \ref{thm:main1}.
 
\section{Algorithms}
In this section, we outline our broad approach and give the main algorithms we propose. First, we have the following definition of weak and strong supports:
\begin{definition}
The weak support of $f$, denoted by $S$, is defined to be the set of variables that $f$ depends on, i.e., $S= \{ x_i: \exists j ~\mathrm{s.t}~(i,j) \in Q ~\mathrm{or}~ i \in L \}$. Similarly, we define strong support of $f$ to be $\{Q,L\}$.
\end{definition}

In order to achieve time complexity linear in $p$, we divide the objective of learning $f$ into three tasks: 
\begin{enumerate}[itemsep=-2pt]
 \item Learning the weak support.
 \item Learning the coefficients of the linear and quadratic terms formed by the weak support.
 \end{enumerate}
 The last step can be implemented using a log-likelihood optimization in $O(s^2)$ dimensions. In order to achieve sub-quadratic (in $p$) time complexity, it is sufficient to identify the weak support quickly. In fact, it is easy to see once the weak support is identified, the runtime of the other steps will not depend on $p$ but only on $s$. Hence, we focus on identifying the weak support first. 
 
\subsection{A Linear Correlation Test for Binary Variables}
We first consider the model in (\ref{def:Y}) for binary variables. We propose Algorithm \ref{alg:weaksupport} for identifying the weak support which is based on a simple correlation test. Essentially, Algorithm \ref{alg:weaksupport} tests if $\lvert \E[Y(x_k-\mu_k)] \rvert > \epsilon$ or not.  The weak support algorithm runs in time $O(p)$ (Here, $O(\cdot)$ subsumes dimension independent quantities ). 
   
\begin{algorithm}[tb]
\begin{small}
   \caption{Identify Weak Support}
   \label{alg:weaksupport}
\begin{algorithmic}
   \STATE {\bfseries Input:} Data $\mathbf{X}\in \{+1,-1\}^{n\times p}, Y\in \{0,1\}^{n\times 1}$, Threshold $\varepsilon$
   \STATE {Initialize:} $S=\emptyset$.
   \STATE {Estimate $\mu_i=\E[X_i]$ by} $\hat{\mu}_i=\frac{1}{n}\sum_{j=1}^n\mathbf{X}(j,i)$
   \FOR{$i=1$ {\bfseries to} $p$}
   \STATE{Estimate $\rho_i=\E[Y(X_i-\mu_i)]$ by\\ $\hat{\rho}_i=\frac{1}{n}\sum_{j=1}^n Y(j)(X(j,i)-\hat{\mu}_i)$}
   \IF{$\hat{\rho}_i > \varepsilon$} 
   \STATE $S=S\cup \{x_i\}$
   \ENDIF
   \ENDFOR
   \STATE{\bfseries Output:} S
\end{algorithmic}
\end{small}
\end{algorithm}

   \textbf{Sample complexity:}
   Algorithm \ref{alg:weaksupport} works when the number of samples $n$ is large enough, given that $\varepsilon$ is properly chosen. Given an $\varepsilon$ for which the population version succeeds, it can be shown that, if the number of samples is above $8c\log(p) \left(\frac{1}{\varepsilon}\right)^2$ the algorithm sssucceed with polynomially small error probability for any $c>1$. This is due to the fact that the random variable $ Y(X_i-p_i)$ lies in $[-2,2]$ and Hoeffding's inequality can be used to derive strong concentration results. Existence of an $\varepsilon>0$ is shown for the sigmoidal non-linearity while we characterize an $\varepsilon$ for the case when $\sigma$ is piece-wise linear.

\subsection{A Non-linear Correlation Test for Variables with Finite Support}
In the previous sections, we saw intuitively that the non-linearity in the bias of the target variable leaves traces that can be detected through a simple correlation test. We will see in the subsequent sections that this statement is not unconditional: There are certain functions where the correlation test cannot be used to find the relevant variables. In this section, we propose a generalization of the correlation test, which we term as a \emph{non-linear correlation test}. Non-linear correlation test applies a random \emph{non-linear transform} on the features before finding their correlation with the target variable. Hence, the test can exploit \textit{random non-linearity} injected in the test, rather than solely relying on the non-linearity due to nature. Notice that when the variable is binary $\{+1,-1\}$, any nonlinear test is equivalent to some linear test. Therefore variables of interest should be non-binary. Thus, we consider the model in (\ref{def:Y}) for real variables with finite support on the real line. 

Through simulations, first we identify a specific function where the linear correlation test fails to detect one of the variables in the weak support. Later in Section \ref{sec:randomFunction}, we empirically show that the nonlinear correlation test can be used to infer the existence of that variable in the weak support with reasonable number of samples.
\subsubsection{Introducing Randomness through Hashing}
\label{sec:specificFunction}
 We consider variables $x_i$ that take value from a finite set from the alphabet $\mathcal{X} \subset \mathbb{R}$. Let $f$ be a quadratic polynomial over $x_i$'s as in (\ref{def:Y}). We observe $Y \sim \mathrm{Ber} ( \sigma (f(x)))$ where $\sigma$ is a non linearity. The nonlinear correlation test we propose is as follows: First pass $x_i$ through a random non-linearity $Z = g(x_i)$ and perform the correlation test:
 \begin{equation}\label{nonlintest}
  \mathbb{E}\left[Y \frac{(Z- E[Z])}{\mathrm{stddev}(Z)} \right] > \theta.
  \end{equation}
 We choose the non-linearity $g$ to be a randomly chosen hash function (from a hash family) each of which maps the values in $\mathcal{X}$ to the set of integers in $[-U,U]$. Note that, $\mathrm{stddev}(Z)$ is the standard deviation of the transformed random variable $Z=g(x_i)$ for a fixed $g$ over the marginal distribution of $x_i$. Further, the $\mathbb{E}[\cdot]$ is over $Y,\{X_i\}$ and the randomness in choosing the hash $g(\cdot)$. We finally compare the obtained correlation values with some threshold and decide to include it in the weak support. We give a finite sample algorithmic version of this in Algorithm \ref{alg:nonlineartest}.
 
 \begin{algorithm}[tb]
\begin{small}
   \caption{Non-Linear Correlation Test}
   \label{alg:nonlineartest}
\begin{algorithmic}
   \STATE {\bfseries Input:} Data $\mathbf{X}\in \{+1,-1\}^{n\times p}, Y\in \{0,1\}^{n\times 1}$, Threshold $\theta$.
   \STATE  Consider $m$ random hash functions $g_1,g_2 \ldots g_m$ between $\mathcal{X}$ to integers in $[-U,U]$. 
   \STATE {Initialize:} $T=\emptyset$.
     \FOR{$i=1$ {\bfseries to} $p$} 
       \STATE $\mu^{\ell}_{i} \leftarrow \frac{ \sum \limits_{k=1}^{n} g_{\ell} (x_{ki})}{n}$ (empirical mean).
       \STATE $\sigma^{\ell}_i \leftarrow \sqrt { \frac{ \sum \limits_{k=1}^{n} \left(g_{\ell} (x_{ki}) - \mu^{\ell}_{i}\right)^2 }{n-1}  }$ (empirical standard deviation).
      \STATE $C_i \leftarrow \frac{1}{m} \left( \sum \limits_{\ell=1}^m \left ( \frac{\sum \limits_{k=1}^n  y_k  (g_{\ell} (x_{ki}) - \mu^{\ell}_{i}) }{ n\sigma^{\ell}_i} \right) \right)$
      \STATE If $C_i> \theta$, $T \leftarrow T \bigcup {i}$.
   \ENDFOR
   \STATE{\bfseries Output:} $T$.
   \end{algorithmic}
\end{small}
\end{algorithm}

\subsubsection{A simple example:}
Let $X_1,X_2,X_3 \in \{ -2,-1,1,2 \}$. Consider the function $f=\alpha(X_1-\mu_1)(X_2-\mu_2)+\beta (X_3-\mu_3)+\gamma$. Let Y be the binary variable $Y\in\{-1,1\}$ were $\pr{Y=1} = 1/(1+\exp(-f))$. Using a simulation-aided search, we identify a set of coefficients $\alpha,\beta,\gamma$ and a set of marginal distributions for $X_1,X_2,X_3$ such that the correlation $E[(X_1-\mu_1)Y]$ is close to zero. See Appendix \ref{sec:parameterVals} for the selected parameter values. Later, we pass $X_1$ through a non-linearity $Z = g(X_1)$ and perform the correlation test in (\ref{nonlintest}) via Algorithm \ref{alg:nonlineartest}. We choose the non-linearity to be a random hash function that maps the values $\{-2,-1,1,2\}$ uniformly randomly to integers in the range $[-1000,1000]$. We observe that even for 500 samples, in 99 out of 100 trials, $X_1$ has the nonlinear correlation higher than every irrelevant variable. For the number of samples larger than 1000, $X_1$ always has the highest nonlinear correlation.
 
Later, we show empirically that a large fraction of the variables in the weak support can be detected with the non-linear correlation test when the number of samples is large in the experimental section.

\section{Analysis of Correlation Tests - Boolean Case}
\subsection{Analysis: Learning the Weak Support}
In this section, we consider the model in (\ref{def:Y}) with boolean variables. We show the following smoothed analysis result for Algorithm \ref{alg:weaksupport}: under mild assumptions on $f$ (needs to satisfy the unique sign property), weak support can be identified in linear time (in $p$) for almost all scaling parameters $\gamma \in [0,1]$ using a simple correlation test: Compare $E[Y(X_i-\mu_i)]$ with $0$. The test succeeds when $E[Y(X_i-\mu_i)]$ is $0$ when $Y$ is independent of $x_i$ and not zero otherwise. 

The function $f$ depends on $r$ binary variables. Therefore, it can take at most $2^r$ values. Let the distinct \textit{non zero} values that the function takes be $v_1,v_2 \ldots v_m$ (where $m \leq 2^r$). Let $v_{m+1}=0$. This value is used if the function takes value $0$. 

\subsubsection{Sufficient conditions}
 Without loss of generality, let us assume that the function $f$ depends on the first $r$ variables, i.e. $\{x_1 \ldots x_r \}$. We will give conditions on $f$ and $\mu_i$'s such that the correlation with the first variable $\E[Y(x_1-\mu_1)] \neq 0$. 
 
\begin{definition} 
 Define $G=(V,E)$ to be a graph on $r$ vertices associated with the function $f$ such that an edge $\{i,j\} \in E$ if $(i,j) \in Q$. Let $C$ be the connected component in the graph $G$ containing the variable $x_1$.
\end{definition}
\begin{definition}
Function $f$ is said to have a \textit{unique sign property} with respect to a value $v_i$ if there is no $v_j \neq v_i$ such that $v_j=-v_i$ and when $f=v_i$, it fixes the sign of all parities $\{x_ix_j\}_{(i,j) \in Q}, \{x_j\}_{j \in L}$. 
\end{definition}

Another way of stating the unique sign property is by saying $\lvert f \rvert = \lvert v_i \rvert$ fixes the sign of all parities $\{x_ix_j\}_{(i,j) \in Q}, \{x_j\}_{j \in L}$.

We show that the idealized test succeeds when the bias probabilities are away from $0,1$ and $1/2$. There are two mutually exclusive cases in the main result. In one case, we need the function to satisfy the \textit {unique sign property} for one non-zero value and in the other case we need it to satisfy the unique sign property for all its values.  
\begin{theorem}\label{thm:main1}
  Let $p_i \in (\delta,1/2-\delta) \cup (1/2+\delta,1-\delta),~\forall i$ for some $\delta>0$. Let $Y,\mathbf{x}$ follow the model given in (\ref{def:Y}). Let $\sigma(\cdot)$ be the sigmoidal function. Then, $E[Y(X_1-\mu_1)] =0$ when $Y$ is independent of $X_1$. Also, $E[Y(X_1-\mu_1)] \neq 0$ for all but finite values of $\gamma \in [0,1]$ when $Y$ is dependent on $X_1$ and when one of the following conditions are satisfied:
  \begin{enumerate}
      \item $f$ has a linear term in $x_1$ or has a linear term $x_i$ which is in the same component $C$ in $G$ as $x_1$ . $\exists v_i \neq 0:$ such that $\lvert f \rvert = \lvert v_i \rvert$ implies a \textit{unique sign} for all the parities $\{ x_ix_j \}_{(i,j) \in Q},\{x_j\}_{j \in L}$.
      \item $f$ has no linear term in $x_1$ and $C$ contains no variable with a linear term. $\forall$ values $v_i \neq 0:\lvert f \rvert = \lvert v_i \rvert$ implies a \textit{unique sign} for all the parities $\{ x_ix_j \}_{(i,j) \in Q},\{x_j\}_{j \in L}$.
  \end{enumerate}
\end{theorem}
\begin{proof}
 The proof is relegated to the Appendix. The proof follows the basic outline provided in Section \ref{whyhelp}.
 \end{proof}
 \textit{Biases need to be away from $1/2$:} We point out that if all bias probabilities are $1/2$, under the conditions in Case $2$, the test does not succeed. This is because the quantity that appears in the proof, i.e. $g(v_i)$ is $0$ for all $v_i$. This means that the bias probabilities being away from $1/2$ is essential for the final result.

\textbf{Note:} The condition that all the coefficients are in general position, i.e.  $c s_c+  \sum \limits_{(i,j)\in Q} \beta_{i,j} s_{i,j} + \sum \limits_{j \in L}\alpha_{j} s_{j} \neq 0$ for any set of signs $s_{i,j},s_j,s_c \in \{0,1,-1\}$, ensures the unique sign condition stated for all values of the function in the Case $2$. For Case $1$, we have used the fact that there is no $j:v_j=-v_i$. This can be ensured if the constant $c$ is in general position with respect to other coefficients $\{\beta_{i,j} \}, \{\alpha_j \}$. In fact, we only need, $c \neq \sum \limits_{(i,j)\in Q} \beta_{i,j} s_{i,j} + \sum \limits_{j \in L}\alpha_{j} s_{j}$ for any set of signs $s_{i,j},s_j \in \{0,1,-1\}$.

\textbf{Remark:} The above theorem works for a broad class of non-linearities $\sigma (\cdot)$ that is complex analytic and has more than $2^s$ non zeros terms in its taylor expansion. Any transcendental function is a special case that satisfies our properties. 

\begin{figure*}[ht!]
\centering
\begin{subfigure}[b]{0.3\textwidth}{\label{fig1.1}\includegraphics[width=\textwidth]{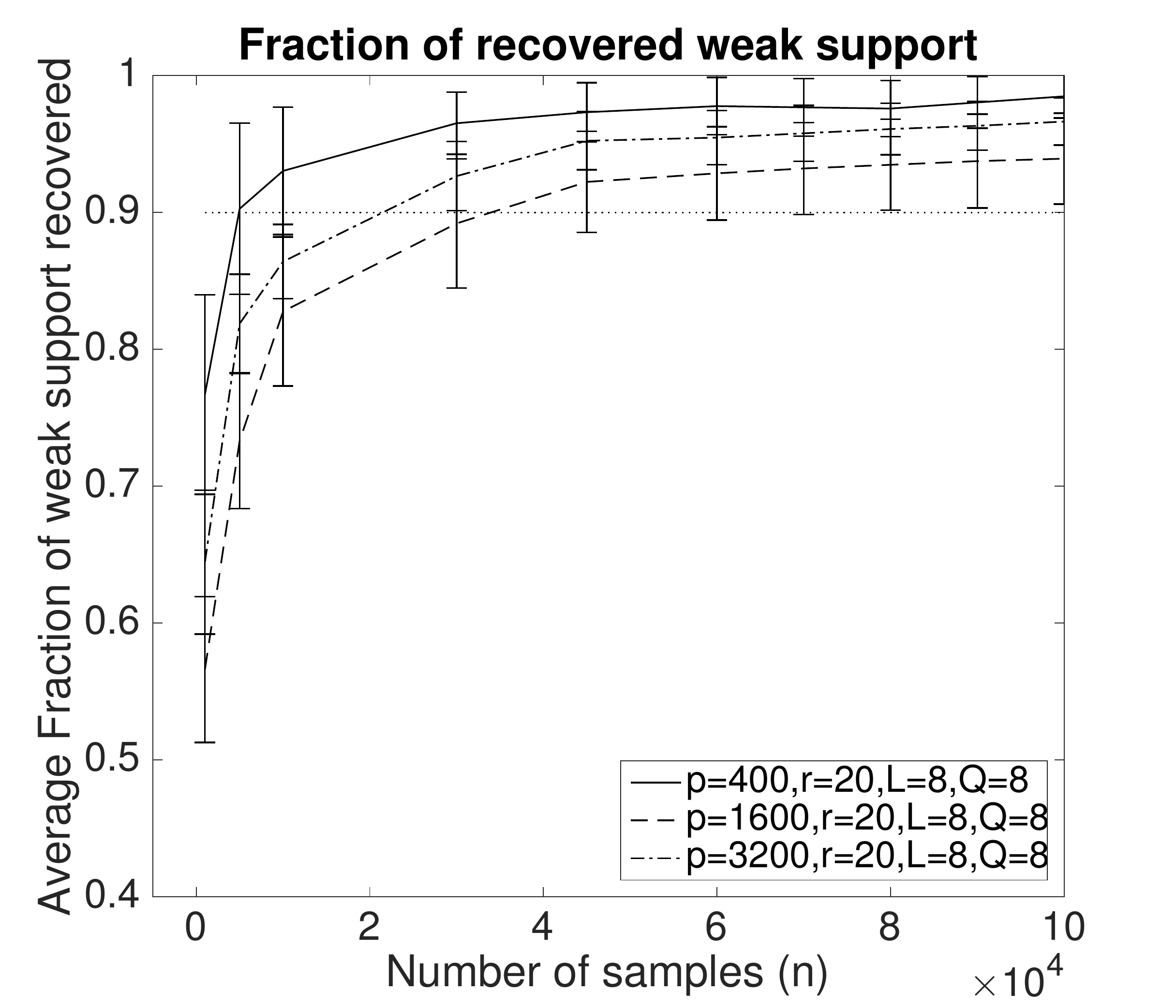}}
\end{subfigure}
\begin{subfigure}[b]{0.3\textwidth}{\label{fig1.2}\includegraphics[width=\textwidth]{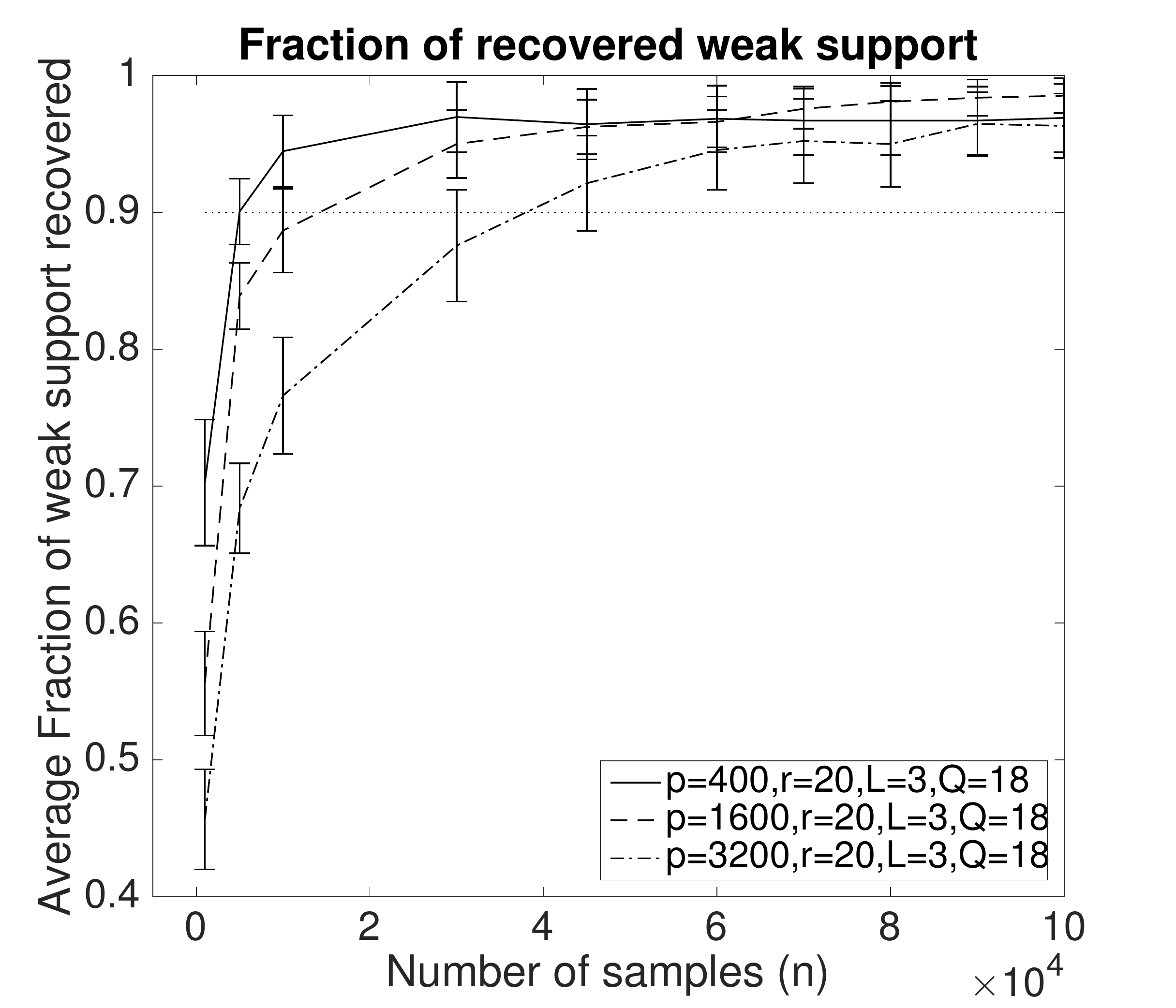}}
\end{subfigure}
\begin{subfigure}[b]{0.3\textwidth}{\label{fig1.3}\includegraphics[width=\textwidth]{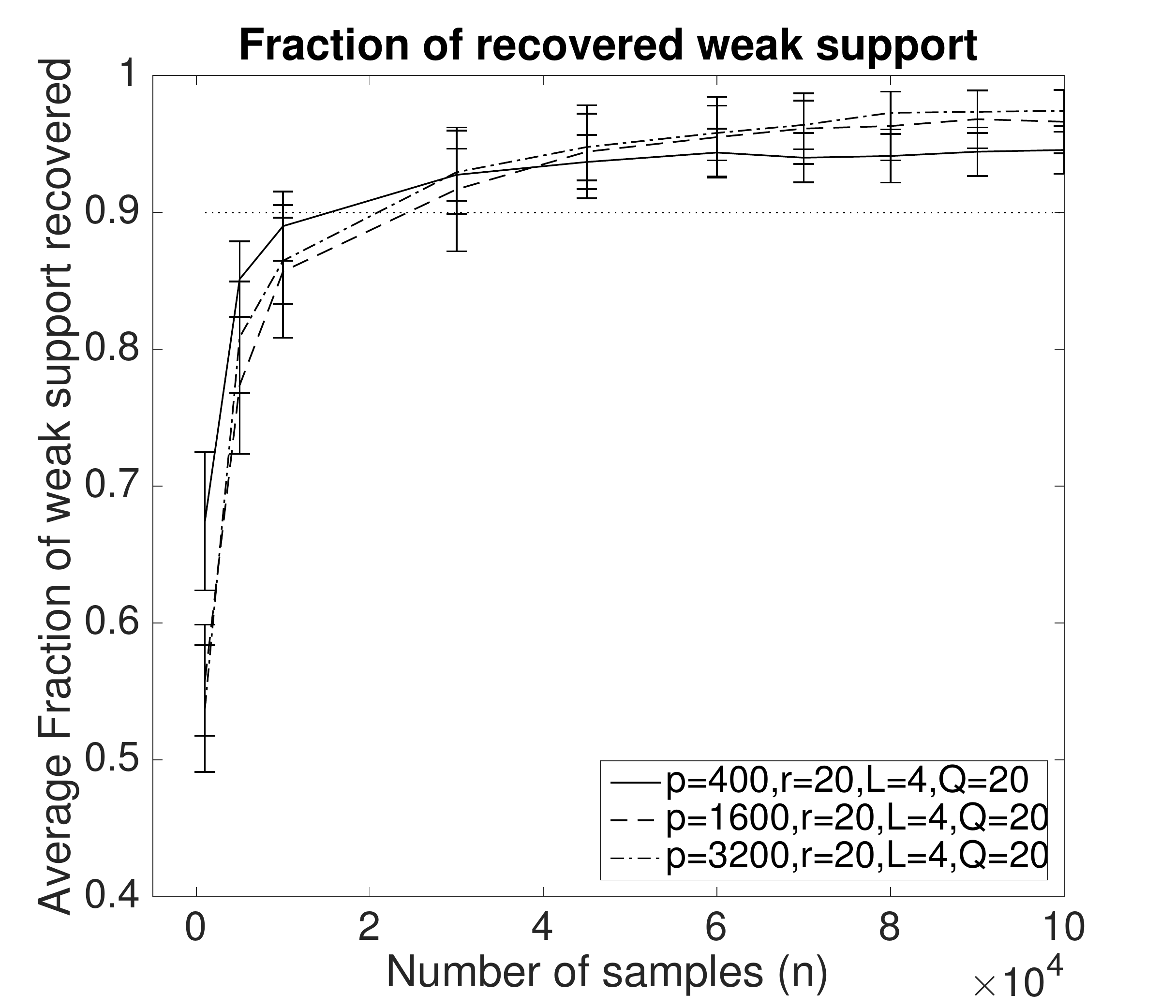}}
\end{subfigure}
\caption{Fraction of times the weak support (relevant variable set) is recovered. $Y,\mathbf{x}$ follow the sparse quadratic logistic regression model in (\ref{def:Y}) with $\sigma(\cdot) = \frac{\exp(\cdot)}{1+\exp(\cdot)}$. The plots show that average fraction of the weak support recovered averaged over the random choices of the scaling parameter $\gamma$ and the bias-probabilities.  The $90 \%$ line for $R=20$ variables refers to the event: All but at most $2$ variables in the weak-support are recovered. The ratio of $L/Q$ varies and $p$ varies from $400-3200$. The number of samples for these cases is $\sim 10^4$. This illustrates that we need significantly fewer samples compared to what our theory predicts (Theorem \ref{samplecomp}). }
\label{fig:weakaccuracy1}
\end{figure*}

\section{Weak Support Learning: Finite Sample Complexity Results}
In this section, we again consider boolean variables. We will show that when $Y$ depends on $x_1$ and when $f$ satisfies some technical conditions , $\lvert \mathbb{E}[Y(x_1-\mu_1)] \rvert \geq \varepsilon$ for values of $\gamma$ in a set of measure which is significant in an interval on the real line. Here, we characterize $\varepsilon$ as a function of other parameters. This establishes sample complexity requirements for Algorithm \ref{alg:weaksupport}. We prove this result when the non-linearity $\sigma(\cdot)$ is a  piecewise linear function. 
\begin{align}\label{nonlinear}
 \sigma(x) & = \frac{1}{2}+\frac{1}{2} L(x) \nonumber \\
  L(x)        & = \left \{ 
    \begin{array}{c}
      1,~~~~~~~ x \geq 1 \\
      x, ~~~ -1 \leq x \leq 1 \\
     -1,~~~~~~~ x \leq -1
    \end{array}    
  \right\}
\end{align}
The rest of the assumptions remain the same as in ($\ref{def:Y}$). Let $v_1, \ldots v_m$ be the non zero values the function $f$ takes. For the purpose of the following theorem, let $\delta^{''}=\delta^{2r} \min (2^{\delta'/2}-1, 1- 2^{-\delta'/2})$ where $\delta'= \left\lvert \log \left(\frac{2}{1-2\delta}\right)\right\rvert$ for some constant $\delta >0$.

\begin{theorem}\label{samplecomp}
 Let the non-linearity $\sigma$ be the function in $(\ref{nonlinear})$.
 Assume that the coefficients of the function $f$ satisfy the following stronger general position property: $ \lvert \sum \limits_{(i,j) \in Q} \beta_{i,j} s_{i,j} + \sum \limits_{j \in L} \alpha_j s_j + c s_c \rvert > \epsilon$ for all set of signs 
 $s_{i,j},s_j,s_c \in \{0,+1,-1\}$. Further, let the bias probabilities $p_i \in (\delta,1/2-\delta) \cup (1/2+\delta, 1- \delta)$. When the function is independent of $x_1$, $\lvert \E[Y(X_1-\mu_1)] \rvert =0$. When the function is dependent on $x_1$, we have the following cases:
 \setlist{nolistsep}
\begin{enumerate}
 \item If the component $C$ containing $x_1$ in graph $G$ has a linear term in $f$, then $\lvert \E[Y(X_1-\mu_1)] \rvert  >   C_1 \frac{\epsilon^2 \delta^{r+2}}{b^2s^2} $ for $\gamma$ in a set of measure $C_2\frac{m\epsilon}{b^2s^2} $ in the open interval $(0, \frac{1}{\lvert v_m \rvert}) \subseteq (0, \frac{1}{\epsilon})$, where $C_1=\frac{1}{32}$ and $C_2=\frac{3}{8}$
 \item If the component $C$ containing $x_1$ in graph $G$ has no linear term in $f$, then $\lvert \E[Y(X_1-\mu_1)] \rvert  >   C_1 \frac{\epsilon^2 \delta^{2}\delta^{''}}{b^2s^2} $ for $\gamma$ in a set of measure $C_2\frac{m\epsilon}{b^2s^2} $ in the open interval $(0, \frac{1}{\lvert v_m \rvert}) \subseteq (0, \frac{1}{\epsilon})$, where $C_1=\frac{1}{32}$ and $C_2=\frac{3}{32}$.
\end{enumerate}
\end{theorem}

\begin{proof}
 A full proof appears in the Appendix.  We give a brief sketch of the proof. The correlation with a relevant variable $x_i$ results in a piecewise linear function in the scaling parameter $\gamma$. Therefore, if the slope of most of these piecewise linear parts can be shown to be large, then a significant portion of the curve is away from the axis and implies a lower bound on the absolute value of the correlation. Although $\sigma$ is not analytic, we get a Vandermonde-like system of equations connecting influence variables ($g(v_i)$ as in Section \ref{whyhelp}) to the slopes of the piecewise linear parts. Using a looser version of the unique sign property, one can show that a significant fraction of the influence variables are large in magnitude proving the result.
 \end{proof}

\section{Experiments}
\subsection{Synthetic Experiments: Weak Support Recovery }
 We illustrate the effectiveness of the correlation test in Algorithm \ref{alg:weaksupport} using synthetic experiments run on models that obey the setup in (\ref{def:Y}). We describe the experimental setup in Fig. \ref{fig:weakaccuracy1}. We fix the quadratic polynomial $f$ where $\lvert L \rvert$ linear terms are uniformly chosen from a fixed set of $r$ variables. Similarly, $\lvert Q \rvert$ bilinear terms are chosen uniformly randomly from the terms formed by $r$ variables. The coefficients of the quadratic polynomial is chosen from a uniform distribution over $[0.1,1]$. The number of relevant variables is $r=20$ and $s \sim 15-25$ for the experiments.  For this fixed polynomial, we randomly generate $\gamma$ uniformly sampled from $[1,15]$ and for every gamma we generate several bias-probabilities uniformly randomly distributed in $[0.1,0.4] \bigcup [0.6,0.9]$. Here, $\delta$ is taken to be $0.1$. For a fixed polynomial, we average over $~140$ random initializations of $\gamma$ and bias-probabilities. We average over $20$ iterations with random choices of polynomials.  We compute the fraction of times the weak support is recovered which is the ratio of the size of the part of the support set recovered to the the total size of weak support

 \textit{Key Observations:} The $90 \%$ line for $R=20$ variables refers to the event: All but at most $2$ variables in the weak-support are recovered. In Fig. \ref{fig:weakaccuracy1}, we find that for different ratios of $\lvert L \rvert$ and $\lvert Q \rvert$, roughly $~10^4$ samples are required to cross this threshold when $p$ is in a few thousands. This is much smaller than the finite sample complexity results in Theorem \ref{samplecomp}. This suggests that in practice, the correlation tests work with fewer samples than what is predicted by the theory. 
 
 \subsection{Experiments with real data-sets: Weak Support Recovery}
 \textbf{Objective:} We evaluate the effect of adding quadratic terms formed by the binary features in the weak support (of small fixed sparsity) identified by Algorithm \ref{alg:weaksupport} to existing binary features in a real data set. We show that this boosts the performance of a standard classifier that works with a linear function of the binary features (standard logistic regression and Linear SVM classifier).
 
 \begin{table*}[ht]
\centering
\begin{tabular}{|clc|c|c|}
\hline
 Classifier & AUC Score  & Sparsity after CV & Running Time \\
 \hline
 \textit{Logistic + $\ell_1$  } & 0.890357 & N/A &  6.197472s \\
 \textit{Logistic + $\ell_1$ + quad terms} & \textbf{0.909252} & 8 & 12.022856s  \\
 \textit{LSVM + $\ell_1$} & 0.822040 & N/A & 3.986561s \\
 \textit{LSVM+ $\ell_1$ +quad terms } & \textbf{0.881608} & 10 & 32.880251s \\
 \hline
\end{tabular}
 \caption{Comparison of standard classifiers on the \textit{Dorothea} dataset with and without addition of interaction terms using our correlation test: We take two standard classifiers -Linear SVM and Logistic Regression both with $\ell_1$ regularization. We train them on the binary feature matrix as is. Subsequently, we use the correlation test from Algorithm \ref{alg:weaksupport} to identify weak support (controlled by a sparsity parameter $k$). We create new features by multiplying binary features in the weak support pairwise and train the classifiers again to compare.} 
 \label{my-table}
\end{table*}

 \textbf{Dataset:} We performed our experiments on the \textit{Dorothea} data set obtained from the UCI Machine Learning Repository.  \footnote{The URL for this data set is: \textit{https://archive.ics.uci.edu/ml/datasets/Dorothea}. This data set was contributed to this repository by DuPont Pharmaceuticals Research Laboratories and KDD Cup 2001.} This data set arises from a use case in drug design. In drug design, one wants a compound that can efficiently bind to a target receptor for the drug to be effective. The data set contains 1150 samples (800 train + 350 validation samples).  Each sample is a set of $100000$ binary features of a compound. The data set has about $50000$ random features added to $50000$ genuine ones. The target is a binary variable indicating if the compound binds or not. 
 We used the training samples for cross validation for our models while we tested on the validation set and we report the test AUC (Area under the Receiver Operating Characteristic Curve (ROC)) scores.  
 
\textbf{Algorithms Compared:} We first perform a correlation test using Algorithm \ref{alg:weaksupport} and then we add all possible quadratic terms arising from the output weak support\footnote{When we do a correlation test using Algorithm \ref{alg:weaksupport}, we normalize the correlation value by the standard deviation of that binary feature. Then, we rank the obtained normalized correlation values and pick the top $k$ features.} to the existing features. We train a standard classifier that works with these expanded set of features. We compare the performance of the same classifier run without any feature addition. The regularization parameters $C$ in all our experiments used for cross validation is such that $\log_{10} (C)$ takes $15$ uniformly space values in the interval $[-4,4]$. We do a $4$ fold cross-validation at the training stage. More specifically, we compare the following four algorithms: 
 \begin{enumerate}[noitemsep]
  \item \textit{LSVM + $\ell_1$ regularization}: This is a standard linear support vector machine with $\ell_1$ regularization. 
  \item \textit{LSVM + $\ell_1$ regularization + quadratic terms}: Quadratic terms arising out of the $k$ top features according to the correlation test are added. For cross-validation, $k$ is chosen to lie in $\{3,5,8,10,15\}$.
  \item \textit{Logistic Regression + $\ell_1$ regularization } This is the standard logistic Regression with $\ell_1$ regularization. 
  \item \textit{Logistic Regression +  $\ell_1$ regularization + quadratic terms}: Quadratic terms arising out of the $k$ top features according to the correlation test are added. For cross-validation, $k$ is chosen to lie in $\{3,5,8,10,15\}$.
 \end{enumerate}
 \textbf{Results:} We summarize the results in Table \ref{my-table}. We find a $0.02$ improvement in the AUC score for logistic regression and a $0.06$ improvement in the AUC score for Linear SVM. The running times with or without feature addition are comparable. We like to note that adding extra features to an already large feature matrix is a time consuming operation and time required for this step is included in the timing analysis.

\subsection{Synthetic Experiments: Nonlinear Correlation Test for Variables with Finite Support}
\label{sec:randomFunction}
In this section, we provide empirical evidence that the non-linear correlation (Algorithm \ref{alg:nonlineartest}) identifies the weak support of a function dependent on real variables with finite support. We consider the following randomly chosen function: $f=\sum_{(i,j\in S\times S)}\alpha_{i,j}X_iX_j$, where $X_i\in\mathcal{X}=\{-2,-1,1,2\},\forall i$. The uniformly randomly chosen subset $S\subseteq [p]$ is the weak support of $f$, and $\alpha_{i,j}$ are i.i.d samples from a uniform random variable in the interval $[-1,1]$. The probability mass function of each variable $X_i$ is chosen uniformly randomly over the simplex in 4 dimensions. Hence, $f$ has all the quadratic terms, where the coefficient of each term is selected uniformly and independently in the range $[-1,1]$. We consider the case $p=1010$, and $\lvert S\rvert=10$.  We perform the hashing-based nonlinear correlation test to obtain $C_i$ given in Section \ref{sec:specificFunction}. Later, we choose the candidate weak support set $T$ by including every variable $X_i$ which is among the top 20 variables with highest  correlation values $C_i$. We compute the fraction of the true weak support that is contained in $T$, i.e., $\frac{\lvert S\cap T \rvert}{\lvert S \rvert}$. This fraction is averaged over 100 randomly sampled functions and datasets, and is reported in Table \ref{table:randomFunction} for varying number of samples. As observed, with enough samples, we can recover a large fraction of the variables in the weak support on average. Increasing the number of used hash functions improves the performance, although it incurs some computational cost. 

 \begin{table}[t]
 \centering
\caption{Average fraction of the weak support recovered vs samples for the non-linear correlation test. Nonlinear correlation of each variable is the average of the correlation obtained using 10 random hashes. Weak support recover rate is the fraction of the weak support that is in the top 20 variables with highest nonlinear correlation.} 
\begin{tabular}{|clc|c|c|c|}
\hline
 Samples & $500$ & $1000$ & $5000$ & $10000$ \\
 \hline
 WS  &  &  & &   \\
  Rec. Rate  & $0.346$ & $0.488$  & $0.732$  & $0.822$  \\
 \hline
\end{tabular}
\label{table:randomFunction}
\end{table}

\section{Conclusion}
 We propose correlation tests to recover the set of relevant variables for the sparse quadratic logistic regression problem with real covariates with finite support. When the variables are all binary, the correlation test is a simple linear correlation test. We show that the non-linearity inherent in the problem helps the correlation test to succeed. Further, we propose a nonlinear correlation test that involves transforming covariates through hashing before performing correlation for the non-binary case. We show the effectiveness of our methods through a number of theoretical as well as experimental results.
 
\newpage
\bibliographystyle{plain}
\bibliography{quadbib}

\clearpage
\appendix
\section{Appendix}
\subsection{Proof of Theorem \ref{thm:main1} }

  This section provides the proof of the theorem. We first develop some notation, state some intermediate Lemmas and then proceed to the main part of the proof.
  
  It can be easily seen that: 
\begin{align}\label{simplifycorr}
\E[Y(X_k-\mu_k)] &= \E [ \E [Y(X_k-\mu_k) | \mathbf{X}] ] \nonumber \\
\hfill                      &= \E[(X_k-\mu_k) \E [Y | \mathbf{X}] ] \nonumber \\
 \hfill                     & =  \E [\sigma(\gamma f(\mathbf{X})) (X_k-\mu_k)]
\end{align}

Let $E_0=1$ if the function takes value $0$. Let $\mathrm{Pr}(\mathbf{x}: f(\mathbf{x})=v_i | x_k=+1) - \mathrm{Pr}(\mathbf{x}: f(\mathbf{x})=v_i | x_k=-1) = g(v_i))$.
\begin{align}
&\E[Y(X_k-\mu_k)]   =  \sum \limits_{i=1}^{m+1} \sigma(\gamma v_i)  [\mathbf{1}_{i \leq m}+\mathbf{1}_{i=m+1,E_0=1}] \nonumber\\
&\hdots \left[ p_k (1-\mu_k) \mathrm{Pr}(\mathbf{x}: f(\mathbf{x})=v_i | x_k=+1) \right. \nonumber\\
            \hfill   & \left. \ldots - (1+\mu_k)(1-p_k)  \mathrm{Pr}(\mathbf{x}: f(\mathbf{x})=v_i | x_k=-1) \right] \nonumber \\
   \hfill                  &= 2p_k(1-p_k) \left[ \sigma(0) g(0) \mathbf{1}_{E_0=1}+ \sum \limits_{i=1}^m \sigma(\gamma v_i)  g(v_i) \right] \label{eq:sigmoidsum}              
\end{align}
The last equality is because $p_k (1-\mu_k) = (1+\mu_k)(1-p_k) = 2p_k(1-p_k) $. 

Let us assume that the non-linearity is sigmoidal as in $(\ref{def:Y})$. This means that $\sigma(x)= \frac{1}{2}+\frac{1}{2}\tanh\left(x/2 \right)$.  There is a domain $U$ (an open set) containing the real line in the complex plane such that $\tanh(\cdot)$ is an analytic function over $U$ \cite{gamelin2001complex}. This implies that $\left[ \sigma(0) g(0) \mathbf{1}_{E_0=1}+\sum \limits_{i=1}^m \sigma(\gamma v_i)  g(v_i) \right] $  is an analytic function in a domain $U'$ (finite linear combinations of scaled version of analytics functions is analytic in an open neighborhood) containing the real line in the complex plane. It is well known that: 
\begin{theorem}
 \cite{gamelin2001complex} If D is a domain on the complex plane and f(z) is a complex analytic function on $D$ that is not identically zero, then the zeros of $f$ in $D$ are isolated.
\end{theorem}

\begin{theorem}\label{littlewood}
\cite{erdos1945lemma} (Littlewood-Offord Theorem) Consider the linear sum $\sum \limits_{k \in [n]} \alpha_i b_i$ where the joint distribution of $\{b_i\}$ is uniform on $\{-1,1\}^n$ and $\alpha_i \in \mathbb{R}$. If $ \lvert \alpha_i \rvert > \nu$, then 
   \begin{align}
    \mathrm{Pr}\left(\sum \limits_{k \in [n]} \alpha_i b_i \in (-\nu/2, \nu/2) \right) \leq \frac{ {\binom{n}{\lfloor n/2 \rfloor}}}{2^n} \approx \frac{1}{\sqrt{n}}
   \end{align}
\end{theorem}

A set $E$ is said to be isolated if for every point $p \in E$, there is an open ball around $p$ that contains no element from $E$ other than p. This implies on the real line when $z=\gamma \in \left[0,1\right]$, the following is true: 
\begin{lemma}\label{lem:zerocond}
When $\sigma(\cdot)$ is an analytic in a domain containing the real line, exactly one of the following is true:
\begin{enumerate} [itemsep=-3pt]
  \item $\left[\sigma(0) g(0) \mathbf{1}_{E_0=1}+\sum \limits_{i=1}^m \sigma(\gamma v_i)  g(v_i) \right] =0 ,~ \forall \gamma \in [0,1] $. 
  \item $\left[\sigma(0) g(0) \mathbf{1}_{E_0=1}+\sum \limits_{i=1}^m \sigma(\gamma v_i)  g(v_i) \right] =0$ for finitely many $\gamma$'s in $[0,1]$.
\end{enumerate}
Here, all the variables ($\gamma$ and $v_i$) and functions ($g$) are defined in $(\ref{eq:sigmoidsum})$.
\end{lemma}

\textbf{Note:}If the first condition is avoided, then clearly, $\E[Y(x_k-\mu_k)] \neq 0$ a.s for all $\gamma \in [0,1]$. Clearly, when $Y$ is independent of $x_k$, then the first condition is true. 

We now derive some sufficient conditions on $f$ and the bias probabilities $\{\mu_i\}$ such that the first condition is not true when $Y$ depends on $x_k$. This essentially proves that the idealized test works.

 \begin{proof} [Proof of Theorem \ref{thm:main1}]
   From Lemma \ref{lem:zerocond}, it is enough to make sure that the first condition is not true. Since $\sigma(\cdot)$ is analytic in a domain containing the real line, there is a small interval $(0,\gamma_0) \in [0,1]$ (with $\gamma_0 \leq 1$) such that $\sigma(\cdot)$ has a taylor expansion around $0$. Keeping the sigmoidal function (an odd function with a constant shift) in mind, let us for simplicity assume that the taylor expansion is given by: $\sigma(x) = d_0+\sum \limits_{i=0}^{\infty} d_{2i+1} x^{2i+1}$ where $d_0 \in \mathbb{R}\backslash 0$ and $d_{2i+1} \in \mathbb{R}\backslash 0, ~\forall i $. Therefore, in the open interval $(0, \gamma_1= \min (\min \limits_{v_i \neq 0} \frac{\gamma_0}{ \lvert v_i \rvert},\gamma_0) )$, we have the following expansion: 
   \begin{align} \label{taylorexp}
      \sigma(0) g(0) \mathbf{1}_{E_0=1}+\sum \limits_{i=1}^m \sigma(\gamma v_i)  g(v_i)   = d_0 \left[ \sum \limits_{i=1}^m g(v_i)  + \right. \nonumber \\
       \left. g(0) \mathbf{1}_{E_0=1}\right] +\sum \limits_{j=0}^{\infty} d_{2j+1}  \gamma^{2j+1} \left[ \sum \limits_{i=1}^m v_i^{2j+1} g(v_i)  \right]
   \end{align} 

Suppose, the first condition in Lemma \ref{lem:zerocond} is true. Then clearly,  $\sum \limits_{i=1}^m \sigma(\gamma v_i)  g(v_i)  =0$ in the interval $(0,\gamma_1)$. Therefore, all the terms in $(\ref{taylorexp})$ are zero. $ \sum \limits_{i=1}^m g(v_i)+ g(0) \mathbf{1}_{\{f \mathrm{~takes~} 0\}}  = \E[x_1-\mu_1]=0 $ always. Therefore, the hypothesis that the first condition in Lemma \ref{lem:zerocond} is true implies:
\begin{align}\label{vandeqn}
\left[
\begin{array}{cccc}
v_1 & v_2 & v_3 & \ldots v_m \\
v_1^{3} & v_2^{3} & v_3^3 & \ldots v_m^3 \\
\cdots    & \cdots    & \cdots  & \cdots \\
v_1^{2i+1} & v_2^{2i+1} & v_3^{2i+1} & \ldots v_m^{2i+1} \\
\cdots    & \cdots    & \cdots  & \cdots \\
\end{array}
\right] \left[
\begin{array}{c}
g(v_1) \\
g(v_2) \\ 
\cdots \\
g(v_m) \\
\end{array}
\right] = 0 
\end{align}

All the $v_i$'s are distinct. However, there could be $i \neq j: v_i = -v_j$. It can be easily seen that the number of linearly independent columns is $m - \lvert \{ (i,j): v_i = -v_j \} \rvert$ for any finite truncation with $>m$ row.  This is due to two reasons: 1) The matrix truncated is equivalent to a Vandermonde matrix when each column is diivded by some constant, i.e. $i$th column divided by $v_i$ 2) Since, $v_i$'s are distinct in their values,  there is exactly one pair involving $i$ where $v_i=-v_j$ and only this causes reduction in rank.

For every pair $(i,j):v_i=-v_j$, because of the odd-power progression, the null space contains the following the vector with two ones: $[0~ 0 ~0~ \underset{\mathrm{pos.}~i}1~ \cdots~ \underset{\mathrm{pos.}~j}1 ~ \cdots ~0]^T$. All these null-space vectors have disjoint support and hence orthogonal and linearly independent. This means that the null space is $\mathrm{span}\{ [0~ 0 ~0~ \underset{\mathrm{pos.}~i}1~ \cdots~ \underset{\mathrm{pos.}~j}1 ~ \cdots ~0]^T : \exists (i,j) ~\mathrm{with~} v_i=-v_j\}$. Hence, the vector $[g(v_1) \ldots g(v_m)]^T$ avoid the null space of the matrix in $(\ref{vandeqn})$ when 
\begin{align}\label{suffcond1}
&\exists v_i:g(v_i) \neq 0 \mathrm{~and ~for~ the ~same~} v_i \mathrm{~if~} \exists j:v_j=-v_i, \nonumber \\
&\mathrm{~then~}  g(v_j)  \neq  g(v_i) . 
\end{align}

Under the condition in $(\ref{suffcond1})$, according to Lemma \ref{lem:zerocond}, except for finite number of values of $\gamma \in [0,1]$, the idealized test succeeds since $[g(v_1) \ldots g(v_m)]^T$ avoids the null-space of the matrix in (\ref{vandeqn}). Now, we give some conditions on the function $f$ and the probabilities $p_i$ of the variables when the above condition is satisfied. We will assume that  $p_i \in (\delta,1/2-\delta) \cup (1/2+\delta,1-\delta),~\forall i$ for some $\delta>0$ in the all the cases that follow.

\textbf{Case 1}: $f$ has a linear term in $x_1$. $\exists v_i \neq 0:$ such that $\lvert f \rvert = \lvert v_i \rvert$ implies a \textit{unique sign} for all the parities $\{ x_ix_j \}_{(i,j) \in Q},\{x_j\}_{j \in L}$.

The signs of all the parities, uniquely determine the function value. Therefore, the above condition implies that there is no $v_j:v_j = -v_i$. From condition $(\ref{suffcond1})$, we just need to check if $g(v_i) \neq 0$. Since every parity is fixed when $\lvert f\rvert=\lvert v_i\rvert$, , then $x_1$ is also fixed because $x_1$ has a linear term in $f$. Without loss of generality let us assume that when $f=v_i$, the sign of $x_1$ is fixed to be $1$. Based on (\ref{eq:sigmoidsum}) this implies, $g(v_i)=\mathrm{Pr}(\mathbf{x}: f(\mathbf{x}=v_i)|x_1 = +1)$. Since all the bias probabilities $p_i$ are in $(\delta, 1- \delta)$, $g(v_i) \neq 0$.

\textbf{Case 1':} $f$ has only quadratic terms in $x_1$ but there exists a linear term $x_i$ which is in the same component $C$ as $x_1$ in graph G. $\exists v_i \neq 0:$ such that $\lvert f \rvert = \lvert v_i \rvert$ implies a \textit{unique sign} for all the parities $\{ x_ix_j \}_{(i,j) \in Q},\{x_j\}_{j \in L}$.

Since $x_i$ and $x_1$ lie in the same component $C$ in graph $G$, there is a path $P \subseteq E$ from $x_i$ to $x_1$ containing adjacent edges in sequence. Since every parity is fixed when $f=v_i$, , then $x_i$ is also fixed because $x_i$ has a linear term in $f$. Let $x_i,x_j$ be the first edge on $P$ from $x_i$. Since $x_ix_j$ is fixed and $x_i$ is also fixed, $x_j$ is also fixed. By propagating the signs across the path to $x_1$, the sign of $x_1$ is also fixed. Without loss of generality, let $x_1=+1$ when $f=v_i$. By the argument for the previous case, $g(v_i)=\mathrm{Pr}(\mathbf{x}:f(\mathbf{x}=v_i) | x_1=+1)$ and since all bias probabilities $p_i$ are $\delta$ away from $0$ and $1$, $g(v_i) \neq 0$.

\textbf{Case 2:} $f$ has only quadratic terms in $x_1$ and the connected component $C$ containing $x_1$ does not have any linear term in $f$. $\forall$ values $v_i \neq 0:\lvert f \rvert = \lvert v_i \rvert$ implies a \textit{unique sign} for all the parities $\{ x_ix_j \}_{(i,j) \in Q},\{x_j\}_{j \in L}$.

The assumption on the function leads to two properties: 1) For all $v_i$, $f=v_i$ implies a unique sign of the parities corresponding to $Q$ and $L$. 2) For all $v_i$, $\nexists v_j:v_j \neq -v_i$. Since the component of $x_1$ contains no linear terms, $x_1$ being $+1$ or $-1$ could result in the same sign for all parities. Because of the second property, we just need to show that $g(v_i) \neq 0$ for some $i$. Consider a specific $v_i$. Due to the first property, the signs of parities in $Q$ and $L$ are fixed. Let value of the parity $x_ix_j$ with $(i,j) \in Q$ be $w_{i,j}$ while the value of $x_j$ with $j \in L$ be $y_j$. Now, consider the graph $G$ and the connected component containing $x_1$. When $x_1=+1$ and when $f=v_i$, every variable in the component $C$ takes a specific sign. This is because all parities are fixed. Every path from $x_1$ to another variable $x_i$ fixes $x_i$. This can be seen using arguments identical to Case $2$. For $x_{\ell} \in C$, let the sign taken by $x_{\ell}$ when $x_1=1$ be $z^{+}_{\ell}$. When $x_1=-1$, all the variables $x_i$ in $C$ change sign. Let the signs when $x_1=-1$ be $z_{\ell}^{-}=-z_{\ell}^{+}$.  The other, variables outside $C$ are unaffected by the sign change of $x_1$. For ease of notation, define $q_{\ell}^{+} =(z^{+1}_{\ell} +1)/2$ and $q_{\ell}^{-} = (z^{-}_{\ell} +1)/2$. This means that $q^{+}_{\ell}=0 \Rightarrow q^{-}_{\ell}=1$  and vice-versa. Therefore, 
\begin{align}\label{gvalue}
 g(v_i) & =  \mathrm{Pr}\left( x_ix_j = w_{i,j} ~\forall (i,j) \in Q: i,j \notin C, \right. \nonumber \\
  \hfill & \left. ~ x_j =y_j ~\forall j \in L \right)* \left[ \mathrm{Pr} \left( x_{\ell} = z_{\ell}^{+1} ~\forall \ell \in C-\{1\} \right)\right. \nonumber \\
    \hfill & \left. - \mathrm{Pr} \left( x_{\ell} = z_{\ell}^{-1}~\forall \ell \in C-\{1\} \right)\right] \nonumber \\ 
    \hfill & = \mathrm{Pr}\left( x_ix_j = w_{i,j} ~\forall (i,j) \in Q: i,j \notin C, \right. \nonumber \\
      \hfill & \left. x_j =y_j ~\forall j \in L \right) * \left( \prod \limits_{j \in C-\{1\}} p_j \right) *\nonumber \\
       \hfill &  \left[ \prod \limits_{j \in C-\{1\}} \left( \frac{1-p_j}{p_j} \right)^{1-q^{+}_{j}} -  \prod \limits_{j \in C-\{1\}} \left( \frac{1-p_j}{p_j} \right)^{q^{+}_{j}} \right] 
 \end{align}      

Let,
 \begin{align}
  e(v_i) & = \mathrm{Pr}\left( x_ix_j = w_{i,j} ~\forall (i,j) \in Q: i,j \notin C, \right.\nonumber \\
   \hfill & \left.  x_j =y_j ~\forall j \in L \right)
  \end{align}
Now, $ \lvert e(v_i) \rvert * \left( \prod \limits_{j \in C-\{1\}} p_j \right) \geq \delta^{r}$ since $(1-p_i),p_i \geq \delta, ~ \forall i$. Therefore,
\begin{align} \label{gbound1}
  \lvert g(v_i) \rvert & \geq \delta^r \left \lvert  \prod \limits_{j \in C-\{1\}} \left( \frac{1-p_j}{p_j} \right)^{1-q^{+}_{j}} -  \prod \limits_{j \in C} \left( \frac{1-p_j}{p_j} \right)^{q^{+}_{j}} \right \rvert
\end{align} 

$\frac{1-p}{p}$ is a decreasing function in $p$.  Consider $ \log \left( \frac{1}{p} -1 \right)$ (with base $e$). At $p=1/2,~ \log (1/p-1)=0$. Further, since $p_i,~\forall i$ is away from $0,1$ and $1/2$ by $\delta>0$, $ \infty  > \log (1/p_i -1 ) > \log (\frac{2}{1-2\delta} -1) \mathrm{~or~} -\infty<\log (1/p_i-1) < \log (\frac{2}{1+2\delta} -1) $. Let $\delta' =  \lvert \log (\frac{1+2\delta}{1-2\delta} )\rvert$. Then,
\begin{equation}\label{lowerbndlog}
 \lvert \log (1/p_i-1)\rvert > \delta'.
\end{equation}
Therefore, $(\ref{gbound1})$ becomes:
\begin{align} \label{gbound2}
 \lvert g(v_i) \rvert &\geq \delta^r * \prod \limits_{j \in C-\{1\}} \left( \frac{1-p_j}{p_j} \right)^{q^{+}_{j}}  * \nonumber \\
  \hfill      & \left \lvert  \prod \limits_{j \in C-\{1\}} \left( \frac{1-p_j}{p_j} \right)^{1-2q^{+}_{j}} -  1 \right \rvert  \nonumber \\
   \hfill      & \geq \delta^r \left(\frac{1}{1-\delta}-1\right)^{\lvert C \rvert-1} \left \lvert  e^{\sum \limits_{j} \log \left(\frac{1-p_j}{p_j}\right) \left(1-2q^{+}_{j} \right)} -  1 \right \rvert \nonumber \\
    \hfill    & \geq \delta^{2r} \left \lvert  e^{\sum \limits_{j \in C -\{1\}} \log \left(\frac{1-p_j}{p_j}\right) \left(1-2q^{+}_{j} \right)} -  1 \right \rvert 
\end{align}  

Now, the variables $ \left(1-2q^{+}_{j} \right)$ are binary variables which take values in $\{1,-1\}$. $v_i$ decides the parities uniquely and hence $q^{+}_{j}$ is unique. $f$ takes values $v_i$ and possibly $0$. There are totally $n=m$ or $n=m+1$ values depending on whether $f$ takes the value $0$ or not. Suppose a value is drawn among the $n$ values randomly uniformly, we show that this induces a joint distribution of $\{q^{+}_{j}\}_{j \in C- \{1\}}$ which is uniform.

The component $C$ does not contain any linear term. Therefore, the signs of parities in $C$ and the rest of the parities are independent. Consider a spanning tree $T$ of $C$. It has $|C|-1$ (bilinear) parities and it has $|C|$ variables. When a parity $x_ix_j = +1$, then it corresponds to the linear equation $y_i +y_j = 0$ over the binary field $\mathbb{F}_2$ where $y_i$ and $y_j$ are $0,1$ variables corresponding to $x_i$ and $x_j$ respectively. $+$ is an XOR operation over the field. $1$ is mapped to $0$ and $-1$ is mapped to $1$. Now, the set of equations corresponding to the parities in $T$ ($|C|-1$ in number) are linearly independent over the binary field $\mathbb{F}_2$. Let the linear system be represented as $\mathbf{T}\mathbf{y}$ where $\mathbf{T} \in \mathbb{F}_2^{|C|-1 \times |C|}$. Let us add an extra equation that sets $x_1$. This means that $y_1$ is set to a value. Now, the new set becomes $\mathbf{T}'\mathbf{y}$ where the last row has a $1$ in the $1$st position corresponding to an equation for $y_1$. This corresponds to a full rank set of equations in $|C|$ variables. Now, if $\mathbf{T}'\mathbf{y}=\mathbf{b}$, then $\mathbf{y}=(\mathbf{T}')^{-1} \mathbf{b}$ because $\mathbf{T}'$ is full rank. Now, if the parity values of $T$ along with the value of $x_1$ take uniform random values then the value of every variable in $C$ is also uniform. Clearly, when $x_1$ is fixed to be $1$, the remaining variables are also uniform random. This means that if the parity values in $C$ are uniform random then clearly, the $y_i$ values are uniformly random. This means that joint distribution of $\{q^{+}_{j}\}_{j \in C- \{1\}}$ is uniform.

Now, we show that uniform distribution on the $n$ values of $f$ induces a uniform distribution on the parities of $C$. Every value has a unique sign for the set of parities in $Q,L$. The set of parities in $C$ are independent of the rest of the parities. Therefore, the number of function values when the parities in $C$ are fixed is exactly $\frac{n}{2^{|C|-1}}$ and is exactly the same irrespective of the values of parities for $C$. This also means $n$ is a power of $2$. Therefore, a uniform distribution on the values will induce a uniform distribution on the parities of $C$ which has been shown to induce a uniform distribution on $\{q^{+}_{j}\}_{j \in C- \{1\}}$. 

By $(\ref{lowerbndlog})$, Theorem \ref{littlewood} and the arguments above with respect to a uniform random variable $V$ on the $n$ values of the function, we have: 
 
 \begin{align}\label{eqn:halfbig}
   \mathrm{Pr} ( \lvert g(V) \rvert > \delta^{2r} \min (2^{\delta'/2}-1, 1- 2^{-\delta'/2}) & \geq 1 - \frac{O(1)}{\sqrt{|C|}}) \nonumber \\
    \hfill  \geq 1/2.
 \end{align} 
 
 This means that for more than $\frac{1}{2}$ of the $v_i$'s (one less than this because $f$ could possible take $0$), $g(v_i) \neq 0$.
 
 \end{proof}

\subsection{Proof of Theorem \ref{samplecomp} }
 Let the non zero function values be $\{v_i\}_{i=1}^m$. Let $v_{m+1}=0$ to be used when the function takes value $0$. Let the event that the function takes $0$ be called $E_0$ (a $0,1$) boolean variable.  The strong general position property implies that $\lvert v_i \rvert > \epsilon$ and $\lvert v_i -v_j \rvert > \epsilon$ and $\lvert v_i+v_j \rvert > \epsilon$. Together, this means that $\lvert \lvert v_i\rvert -\lvert v_j\rvert \rvert > \epsilon$. Further, any value (non-zero or zero) of the function determines a unique sign for the parities in $\{Q,L\}$. For any non-linearity $\sigma(\cdot)$, (\ref{eq:sigmoidsum}) holds. Therefore,
 \begin{align} \label{correxpand}
   \E[Y(X_k-\mu_k)] &=  2p_1(1-p_1) \left[ \sigma(0) g(0) \mathbf{1}_{E_0=1}+ \right. \nonumber \\
    \hfill              & \left. \sum \limits_{i=1}^m \sigma(\gamma v_i)  g(v_i) \right] \nonumber \\
       \hfill          & = p_1(1-p_1) \left[ \mathbf{1}_{E_0=1}g(0)+ \sum \limits_{i=1}^m  g(v_i) \right]  \nonumber \\
       \hfill         &   + p_1(1-p_1)  \left[\sum \limits_{i=1}^m L(\gamma v_i)  g(v_i) \right]
 \end{align}   
Again, $\E[x_1-\mu_1] = 0 = p_1(1-p_1) \left[ \mathbf{1}_{E_0=1}g(0)+ \sum \limits_{i=1}^m  g(v_i) \right] $. Let $g'(v_i)=\mathrm{sgn}(v_i) g(v_i)$. Therefore, $(\ref{correxpand})$ becomes:
  \begin{align}\label{piecewiselinear}
    \E[Y(X_k-\mu_k)] &=  p_1(1-p_1)  \left[\sum \limits_{i=1}^m L(\gamma v_i)  g(v_i) \right] \nonumber \\
      \hfill                     & = p_1(1-p_1)  \left[\sum \limits_{i=1}^m L(\gamma \lvert v_i \rvert)  g'(v_i) \right]
  \end{align}
  
  Since, $L$ is a piece wise linear function, the above expression is also piecewise linear with respect to $\gamma \in [0, \infty)$. Without loss of generality, let us assume that $v_i$ is ordered according to the the index $i$ such that $\lvert v_i \lvert < \lvert v_j\rvert, ~\forall i > j )$. In the piecewise linear function in $(\ref{piecewiselinear})$, the function is linear between $\frac{1}{\lvert v_i \rvert}$ and $\frac{1}{\lvert v_{i+1}\rvert}$ $\forall i \geq 1$ and it is also linear between $0$ and $\frac{1}{\lvert v_1 \rvert}$ and linear in $(\frac{1}{\lvert v_m \rvert}, \infty)$. Now, in every linear part, the function can cross $0$ at most once. Hence, if the slope is not too small, the function is larger in absolute except for a very small interval around this zero. We first analyze the number of intervals among $(0, \frac{1}{\lvert v_1 \rvert})$ , $(\frac{1}{\lvert v_i \rvert},\frac{1}{\lvert v_{i+1} \rvert}),~ 1 \leq i\leq m-1$ and $(\frac{1}{\lvert v_m \rvert},\infty)$ where the magnitude of the slope is large. Let $b_0,b_1 \ldots b_{m-1}, b_m$ be the slopes of these intervals. Then, we have the following linear system from the definition of $L(\cdot)$:
  \begin{equation} \label{eqn:slope}
  \left[
  \begin{array}{cccc}
   |v_1| &  |v_2| & \cdots & |v_m| \\
   1      & |v_2|  & \cdots  & |v_m| \\
   1      &   1      &  \cdots & |v_m| \\
   \cdots & \cdots &\cdots & \cdots \\
   1 &  1 &  \cdots & 1  
  \end{array}
  \right]
  \left[ 
    \begin{array}{c}
      g'(v_1) \\
       \cdots \\
       \cdots \\
      g'(v_m) 
    \end{array}
  \right] = \left[ 
    \begin{array}{c}
       b_0 \\
       \cdot \\
       \cdot \\
       b_m
     \end{array}
  \right] 
  \end{equation}
  
 First, we observe that since $\lvert \lvert v_i \rvert - \lvert v_j\rvert  \rvert > \epsilon$ for all $i \neq j$, there is at most one $\lvert v_i \rvert$ such that $\lvert \lvert v_i  \rvert -1 \rvert < \epsilon/4 $. This means that, for all $i$ except one $ \lvert \lvert v_i\rvert -1 \rvert > \epsilon/4 $.
 
 \textbf{Case 1:} When the function $f$ has a linear term in $x_1$ or if the component $C$ containing $x_1$ has a linear term, by the same arguments in Theorem \ref{thm:main1}, $g(v_i) = \mathrm{Pr} \left( \mathbf{x}: f(\mathbf{x}) =v_i | x_1 =+1 \right)$ or $g(v_i)=-\mathrm{Pr} \left( \mathbf{x}: f(\mathbf{x}) =v_i | x_1 =-1 \right)$ for all $i$. Here, every values $v_i$ implies a unique sign for the parities. Clearly, since $p_i,1-p_i \geq \delta$, in this case: $\lvert g'(v_i)\rvert \geq \delta^r,~\forall i$. Now, consider the matrix equation in $(\ref{eqn:slope})$. If, for a particular $i$, $\lvert b_i  \rvert < \frac{\epsilon}{8} \delta^r $, then both $\lvert b_{i-1} \rvert > \frac{\epsilon}{8} \delta^r  $ and $\lvert b_{i+1} \rvert >\frac{\epsilon}{8} \delta^r $. This is because consecutive $b_i$'s differ by $(\lvert v_i \rvert -1) g'(v_i)$ which is bounded below in absolute value. This means that at least $m/2 $ slopes among $b_0, \ldots b_{m-1}$ are large in magnitude. Each interval is of length at least $\lvert \frac{1}{|v_i|} -\frac{1}{|v_{i+1}|} \rvert \geq \frac{\epsilon}{b^2s^2}$. It is easy to verify that in an interval $(\frac{1}{\lvert v_{i}\rvert}, \frac{1}{\lvert v_{i+1}\rvert})$ whose slope is greater than $\frac{\epsilon}{8} \delta^r$, except for a smaller interval of length $\frac{\epsilon}{4b^2s^2 }$ contained in it, the correlation  $\lvert \E[Y(X_k-\mu_k)] \rvert  > \frac{\epsilon^2 \delta^{r+2}}{32 b^2s^2} $ (since $(1-p_1)p_1\geq \delta^2$). Therefore, the test succeeds for $\gamma$ in a set of measure $\frac{3m\epsilon}{8b^2s^2} $ in the open interval $(0, \frac{1}{\lvert v_m \rvert}) \subseteq (0, \frac{1}{\epsilon})$.

 \textbf{Case 2:} Consider the case when $f$ is such that the component containing $x_1$ has no linear term. Further, from the assumptions, $f$ is such that $|f|= \lvert v_i \rvert$ fixes the signs of all parities. Therefore, the value of $g(v_i)$ is identical to Case $2$ in Theorem \ref{thm:main1}. From $(\ref{eqn:halfbig})$, more than $\left(1-\frac{O(1)}{\sqrt{|C|}}\right)$ fraction of the $\lvert g'(v_i) \rvert$'s are larger than $\delta^{''}=\delta^{2r} \min (2^{\delta'/2}-1, 1- 2^{-\delta'/2})$ where $\delta'= \left\lvert \log \left(\frac{2}{1-2\delta}\right)\right\rvert$.
 
 Let ${\cal I}= \{i: \lvert g'(v_i) \rvert > \delta^{''}  \}$. From the discussion in the preceding paragraph, $\lvert {\cal I} \rvert \geq m(1-\frac{O(1)}{\sqrt{|C|}})$. Now, clearly, for every $i \in {\cal I}$ either $b_i$ or $b_{i+1}$ is large ($> \delta^{''} \frac{\epsilon}{8}$) in magnitude because $\lvert |v_i| -1\rvert \lvert g'(v_i) \rvert \geq \delta^{''} \frac{\epsilon}{4}$. Therefore, there can be at least $ \frac{\lvert \cal I \rvert}{4}$ slopes among $b_0, \cdots b_{m-1}$ that is large. Therefore, modifying the constants in the previous case with constants derived above, we have: For $\gamma$ in a set of measure $\frac{3m\epsilon}{16b^2s^2} (1-\frac{O(1)}{\sqrt{|C|}}) $ in the open interval $(0, \frac{1}{\lvert v_m \rvert}) \subseteq (0, \frac{1}{\epsilon})$, $\lvert \E[Y(X_k-\mu_k)] \rvert  > \frac{\epsilon^2 \delta^{2}\delta^{''}}{32 b^2s^2}$.

\section{Parameter Values for the Function in Section \ref{sec:specificFunction}}
\label{sec:parameterVals}
Following are the selected parameters: $\alpha = 2.008,\beta=2,\gamma=2$. Probability mass function for the variables are given by 
\begin{align*}
&X_1:&[ 0.11524836,  0.29707109,  0.28290707,  0.30477349]\\
&X_2:&[0.59194937,  0.20854834 , 0.18259408 , 0.01690821]\\
&X_3:& [0.05462701 , 0.12636817  ,0.01060465 , 0.80840018].
\end{align*}

\section{Strong Support Identifying Algorithms}

We assume that $f$ depends on $r$ variables $\{x_1,x_2,...,x_r\}$. We assume that the weak support learning algorithm has identified the weak support containing the $r$ relevant variables. The objective now is to determine whether $x_i$ appears as a linear term, quadratic term, or both in the function. We study Algorithm \ref{alg:strongsupport} for an important special class of functions.

\begin{algorithm}[tb]
\begin{small}
   \caption{Identify Strong Support}
   \label{alg:strongsupport}
\begin{algorithmic}
   \STATE {\bfseries Input:} Data $\mathbf{X}\in \{+1,-1\}^{n\times p}, Y\in \{0,1\}^{n\times 1}$, Weak Support $S$, Threshold $\theta$.
   \STATE {Initialize:} $T=\emptyset$.
      \FOR{$i=1$ {\bfseries to} $p$}
      \FOR{$j=i+1${\bfseries to} $p$}
   \STATE {Estimate $U_{a,b}(i,j)=\E[Y|X_i=a,X_j=b]$ by \\$\hat{U}_{a,b}(i,j)=\frac{1}{t}\sum_{k=1}^t Y_{a,b}^{(i,j)}(k)$ for $(a,b)\in \{+1,-1\}^2$, where $Y_{a,b}^{(i,j)}=\{Y(l):X(l,i)=a, X(l,j)=b\}$ and $t$ is the size of this set.}
   \IF{$\lvert \hat{U}_{1,1}(i,j) - \hat{U}_{1,1}(i,j) \rvert < \theta$ {\bfseries and} \\\hspace{0.15in}$\lvert \hat{U}_{1,-1}(i,j) - \hat{U}_{-1,1}(i,j) \rvert < \theta$}
   \STATE $T=T\cup \{(x_i,x_j)\}$.
   \ENDIF
   \ENDFOR
   \ENDFOR
   \STATE{\bfseries Output:} $T$ and variables in $S$ but not appearing in $T$.
\end{algorithmic}
\end{small}
\end{algorithm}

We have the following proposition regarding the above algorithm for a very special class of functions.

\begin{proposition}\label{thm:strong}
Assume that each variable $x_i$ appears once in the function $f$, i.e., either as $\alpha_i x_i$ for some $i \in L$ or as $\beta_{i,j}x_ix_j$ for some $i,j \in Q$ but not both. Then, $\mathbb{E}[Y|X_i=1,X_j=1]=\mathbb{E}[Y|X_i=-1,X_j=-1] $ and $\mathbb{E}[Y|X_i=1,X_j=-1]=\mathbb{E}[Y|X_i= -1,X_j=+1]$ when $(i,j) \in Q$ and one of the checks is not true when $(i,j) \notin Q$. 
\end{proposition}
\begin{proof}
In order to prove the theorem, following simple lemmas are useful:
\begin{lemma}
\label{lem:sigmoidBound}
For any $f_1(x), f_2(x)$, we have
\begin{equation}
\sigma((f_1+f_2)(x))>\sigma(f_1(x))\sigma(f_2(x)).
\end{equation}
\end{lemma}
\begin{proof}
\begin{align*}
\sigma(f_1+f_2)&=\frac{\exp(f_1+f_2)}{1+\exp(f_1+f_2)}\\
\sigma(f_1)\sigma(f_2)&=\frac{\exp(f_1+f_2)}{1+\exp(f_1+f_2)+\exp(f_1)+\exp(f_2)}\\
\end{align*}
The proof follows from the fact that $\exp(t)>0,\forall t$.
\end{proof}

\begin{lemma}
\label{lem:sigmoidBound2}
Let $c\geq 0$ be a constant and $f(x)$ be a bounded function. Then we have
\begin{equation}
\E[\sigma(f+c)]\geq\E[\sigma(f)],
\end{equation}
with equality if and only if $c=0$.
\end{lemma}
\begin{proof}
For random variable $x\in [m]$, and $p_i=\Pr(x=i)$,
\begin{align*}
\E[\sigma(f(x)+c)]&=\sum_{i=1}^m\sigma(f(i)+c)p_i\\
&\geq \sum_{i=1}^m\sigma(f(i))p_i=\E[\sigma(f(x))]
\end{align*}
If $c=0$, equality is trivial. If $c>0$, then $\sigma(\alpha+c)>\sigma(\alpha)$ for any $\alpha$ since $\sigma$ is a monotonically increasing function; hence the strict inequality. 
\end{proof}

It is sufficient to prove that $S$ contains all the quadratic terms and no linear terms.

\textbf{Case 1:} $x_i$ and $x_j$ appears as a quadratic term, i.e. $\beta_{i,j}x_ix_j$:

Since $x_i, x_j$ does not appear in another term, by the assumption of the function, expectation does not change unless the value of $x_ix_j$ changes, $\mathbb{E}[Y|X_i=1,X_j=1]=\mathbb{E}[Y|X_i=-1,X_j=-1] $ and $\mathbb{E}[Y|X_i=1,X_j=-1]=\mathbb{E}[Y|X_i= -1,X_j=+1]$ 

\textbf{Case 2:} $x_i$ and $x_j$ appear as linear terms, $\alpha_j x_j$ and $\alpha_i x_i$:

Then, for the first pair of checks in the test, we have $\mathbb{E}[Y|X_i=1,X_j=1] =\E[\sigma(f_0+\alpha_i+\alpha_j)]$ and $\E[\mathbb{E}[Y|X_i=-1,X_j=-1]=\E[\sigma(f_0-\alpha_i-\alpha_j)]$. Defining $h_0=f_0-\alpha_i-\alpha_j$, we have $\E[\sigma(h_0+2(\alpha_i+\alpha_j))]=\E[\sigma(h_0)]$. Lemma \ref{lem:sigmoidBound2} implies $\alpha_i=-\alpha_j$. Similarly, second check implies $\alpha_i=\alpha_j$, which is not possible. Hence, the check fails in this case.

\textbf{Case 3:} $x_i$ appears in a linear term $\alpha_ix_i$ and $x_j$ appears in a quadratic term, $\beta_{j,k} x_jx_k$:

Let us assume that both checks succeed and show a contradiction.
Expand $\E[\sigma(f+\alpha_iX_i+\beta_{j,k}X_jX_k)]$ as
\begin{align*}
&\E[\sigma(f+\alpha_iX_i+\beta_{j,k}X_jX_k)]\\
&=\E[\sigma(f+\alpha_iX_i+\beta_{j,k}X_j)]p(X_k=1)\\
&+\E[\sigma(f+\alpha_iX_i-\beta_{j,k}X_j)]p(X_k=-1)
\end{align*}
First check yields (denoting $p(X_k=1)$ by $p$)
\begin{align*}
&\E[\sigma(f+\alpha_i+\beta_{j,k})]p+\E[\sigma(f+\alpha_i-\beta_{j,k})](1-p)\\
&=\E[\sigma(f-\alpha_i\beta_{j,k})]p+\E[\sigma(f-\alpha_i+\beta_{j,k})](1-p),
\end{align*}
and the second check yields
\begin{align*}
&\E[\sigma(f+\alpha_i-\beta_{j,k})]p+\E[\sigma(f+\alpha_i+\beta_{j,k})](1-p)\\
&=\E[\sigma(f-\alpha_i+\beta_{j,k})]p+\E[\sigma(f-\alpha_i-\beta_{j,k})](1-p),
\end{align*}
Let $h(\mathbf{X})=f(\mathbf{X})-\alpha_i-\beta_{j,k}$. Then the equations become
\begin{align*}
&\E[\sigma(h+2\alpha_i+2\beta_{j,k})]p+\E[\sigma(h+2\alpha_i)](1-p)\\
&=\E[\sigma(h)]p+\E[\sigma(h+2\beta_{j,k})](1-p),
\end{align*}
and
\begin{align*}
&\E[\sigma(h+2\alpha_i)]p+\E[\sigma(h+2\alpha_i+2\beta_{j,k})](1-p)\\
&=\E[\sigma(h+2\beta_{j,k})]p+\E[\sigma(h)](1-p),
\end{align*}
respectively. Rearranging yields
\begin{align*}
&\left(\E[\sigma(h+2\alpha_i+2\beta_{j,k})]-\E[\sigma(h)]\right)p\\
&=\left(\E[\sigma(h+2\beta_{j,k})]-\E[\sigma(h+2\alpha_i)]\right)(1-p),
\end{align*}
and 
\begin{align*}
&\left(\E[\sigma(h+2\alpha_i+2\beta_{j,k})]-\E[\sigma(h)]\right)(1-p)\\
&=\left(\E[\sigma(h+2\beta_{j,k})]-\E[\sigma(h+2\alpha_i)]\right)p,
\end{align*}
Renaming the difference of expectation terms, we have $Ap=B(1-p)$ and $A(1-p)=Bp$. This implies that, $A=B$ and $p=0.5$. But $p=0.5$ is not possible. This contradicts the assumption that the checks succeed.

\textbf{Case 4:} $x_i$ and $x_j$ both appear as distinct quadratic terms, $\beta_{i,k}x_ix_k, \beta_{j,\ell}x_jx_{\ell}$:

Similar to case 3, we assume the checks succeed and show a contradiction. Expand $\E[\sigma(f+\mathbf{\beta_{i,k}X_iX_k}+\mathbf{\beta_{j,\ell}X_jX_{\ell}}]$ by conditioning on both $X_k$ and $X_{\ell}$
\begin{align*}
&\E[\sigma(f+\beta_{i,k}X_iX_k+\beta_{j,\ell}X_jX_{\ell}])]\\
&=\E[\sigma(f+\beta_{i,k}X_i+\beta_{j,\ell}X_j]p(X_k=1)p(X_{\ell}=1)\\
&+\E[\sigma(f+\beta_{i,k}X_i-\beta_{j,\ell}X_j]p(X_k=1)p(X_{\ell}=-1)\\
&+\E[\sigma(f-\beta_{i,k}X_i+\beta_{j,\ell}X_j]p(X_k=-1)p(X_{\ell}=1)\\
&+\E[\sigma(f-\beta_{i,k}X_i-\beta_{j,\ell}X_j]p(X_k=-1)p(X_{\ell}=-1)
\end{align*}
First check yields (denoting $p(X_k=1)$ by $p$, $p(X_{\ell}=1)$ by $q$, and $\bar{p}=1-p, \bar{q}=1-q$)
\begin{align*}
&\E[\sigma(f+\beta_{i,k}+\beta_{j,\ell})]pq+\E[\sigma(f+\beta_{i,k}-\beta_{j,\ell})]p\bar{q}\\
&+\E[\sigma(f-\beta_{i,k}+\beta_{j,\ell})]\bar{p}q+\E[\sigma(f-\beta_{i,k}-\beta_{j,\ell})]\bar{p}\bar{q}\\
=&\E[\sigma(f-\beta_{i,k}-\beta_{j,\ell})]pq+\E[\sigma(f-\beta_{i,k}+\beta_{j,\ell})]p\bar{q}\\
&+\E[\sigma(f+\beta_{i,k}-\beta_{j,\ell})]\bar{p}q+\E[\sigma(f+\beta_{i,k}+\beta_{j,\ell})]\bar{p}\bar{q}
\end{align*}
and the second check yields
\begin{align*}
&\E[\sigma(f+\beta_{i,k}-\beta_{j,\ell})]pq+\E[\sigma(f+\beta_{i,k}+\beta_{j,\ell})]p\bar{q}\\
&+\E[\sigma(f-\beta_{i,k}-\beta_{j,\ell})]\bar{p}q+\E[\sigma(f-\beta_{i,k}+\beta_{j,\ell})]\bar{p}\bar{q}\\
=&\E[\sigma(f-\beta_{i,k}+\beta_{j,\ell})]pq+\E[\sigma(f-\beta_{i,k}-\beta_{j,\ell})]p\bar{q}\\
&+\E[\sigma(f+\beta_{i,k}+\beta_{j,\ell})]\bar{p}q+\E[\sigma(f+\beta_{i,k}-\beta_{j,\ell})]\bar{p}\bar{q}
\end{align*}
Let $h(\mathbf{X})=f(\mathbf{X})-\beta_{i,k}-\beta_{j,\ell}$. Then the equations become
\begin{align*}
&\E[\sigma(h+2\beta_{i,k}+2\beta_{j,\ell})]pq+\E[\sigma(h+2\beta_{i,k})]p\bar{q}\\
&+\E[\sigma(h+2\beta_{j,\ell})]\bar{p}q+\E[\sigma(h)]\bar{p}\bar{q}\\
=&\E[\sigma(h)]pq+\E[\sigma(h+2\beta_{j,\ell})]p\bar{q}\\
&+\E[\sigma(h+2\beta_{i,k})]\bar{p}q+\E[\sigma(h+2\beta_{i,k}+2\beta_{j,\ell})]\bar{p}\bar{q}
\end{align*}
and
\begin{align*}
&\E[\sigma(h+2\beta_{i,k})]pq+\E[\sigma(h+2\beta_{i,k}+2\beta_{j,\ell})]p\bar{q}\\
&+\E[\sigma(h)]\bar{p}q+\E[\sigma(h+2\beta_{j,\ell})]\bar{p}\bar{q}\\
=&\E[\sigma(h+2\beta_{j,\ell})]pq+\E[\sigma(h)]p\bar{q}\\
&+\E[\sigma(h+2\beta_{i,k}+2\beta_{j,\ell})]\bar{p}q+\E[\sigma(h+2\beta_{i,k})]\bar{p}\bar{q}
\end{align*}
respectively. Rearranging yields
\begin{align*}
&(\E[\sigma(h+2\beta_{i,k}+2\beta_{j,\ell})]-\E[\sigma(h)])pq\\
&+(\E[\sigma(h+2\beta_{i,k})]-\E[\sigma(h+2\beta_{j,\ell})])p\bar{q}\\
&=(\E[\sigma(h+2\beta_{i,k})]-\E[\sigma(h+2\beta_{j,\ell})])\bar{p}q\\
&+(\E[\sigma(h+2\beta_{i,k}+2\beta_{j,\ell})]-\E[\sigma(h)])\bar{p}\bar{q}
\end{align*}
and
\begin{align*}
&(\E[\sigma(h+2\beta_{i,k})]-\E[\sigma(h+2\beta_{j,\ell})])pq\\
&+(\E[\sigma(h+2\beta_{i,k}+2\beta_{j,\ell})]-\E[\sigma(h)])p\bar{q}\\
&=(\E[\sigma(h+2\beta_{i,k}+2\beta_{j,\ell})]-\E[\sigma(h)])\bar{p}q\\
&+(\E[\sigma(h+2\beta_{i,k})]-\E[\sigma(h+2\beta_{j,\ell})])\bar{p}\bar{q}
\end{align*}

Renaming the difference of expectation terms, we have $Apq+Bp\bar{q}=B\bar{p}q+A\bar{p}\bar{q}$ and $Bpq+Ap\bar{q}=A\bar{p}q+B\bar{p}\bar{q}$. Summing both equations yield $(A+B)p=(A+B)(1-p)$, which implies either $A=-B$ or $p=0.5$. Assume $A=-B$. Then we have, 
\begin{align*}
&Apq-Ap\bar{q}=-A\bar{p}q+A\bar{p}\bar{q}\\
&\Rightarrow Aq=A(1-q).
\end{align*} 
Since $A=\E[\sigma(h+2\beta_{i,k}+2\beta_{j,\ell})]-\E[\sigma(h)]$ cannot be zero unless $\beta_{i,k},\beta_{j,\ell}$ are zero. Hence either $p=0.5$ or $q=0.5$, which is not allowed by assumption.

\end{proof}

\end{document}